%% file: IEEE-JSAC.tex
\documentclass[lettersize,journal]{IEEEtran}
\newcount\DraftStatus  % 0 suppresses notes to selves in text
\usepackage{comment}
\excludecomment{JournalOnly}  
\includecomment{ConferenceOnly}  
\excludecomment{TulipStyle}
\input{preamble}

\usepackage{tikz,xcolor}% Make Orcid icon
\definecolor{lime}{HTML}{A6CE39}
\DeclareRobustCommand{\orcidicon}{%
	\begin{tikzpicture}
		\draw[lime, fill=lime] (0,0) 
		circle [radius=0.16] 
		node[white] {{\fontfamily{qag}\selectfont \tiny ID}};    \draw[white, fill=white] (-0.0625,0.095) 
		circle [radius=0.007];    \end{tikzpicture}
	\hspace{-2mm}}
\foreach \x in {A, ..., Z}{%
	\expandafter\xdef\csname orcid\x\endcsname{\noexpand\href{https://orcid.org/\csname orcidauthor\x\endcsname}{\noexpand\orcidicon}}
}

\newtheorem{definition}{Definition}
\newtheorem{theorem}{Theorem}

\IEEEoverridecommandlockouts
\usepackage{amsmath,amssymb,amsfonts}
\usepackage{algorithmic}
\usepackage{graphicx}
\usepackage{textcomp}
\usepackage{tabularx}  
\usepackage{booktabs}
\usepackage{xcolor}
\usepackage[numbers]{natbib}

\newtheorem{lemma}{Lemma}

\def\BibTeX{{\rm B\kern-.05em{\sc i\kern-.025em b}\kern-.08em
		T\kern-.1667em\lower.7ex\hbox{E}\kern-.125emX}}

\begin{document}
	
	%
	%=================================================================
	% Preamble which will need to be changed for submission
	%
	%\title{Harnessing Uncertainty for Protection: Quantum Differential Privacy with Quantum Noise
		%}
	\title{Black-Box Auditing of Quantum Model: Lifted Differential Privacy with Quantum Canaries
	}
	
	\author{Baobao Song\orcidA{}, Shiva Raj Pokhrel\orcidB{},~\IEEEmembership{Senior Member,~IEEE}, Athanasios V. Vasilakos\orcidC{},~\IEEEmembership{Senior Member,~IEEE}, Tianqing Zhu\orcidD{},~\IEEEmembership{Member,~IEEE}, Gang Li\orcidE{},~\IEEEmembership{Senior Member,~IEEE}
		\thanks{Received xxx (\textit{Corresponding author: Gang Li}).}
		\thanks{Baobao Song is with the Faculty of Engineering and IT, University of Technology Sydney, Sydney, Ultimo, NSW 2007, Australia~(email: baobao.song@student.uts.edu.au).}%
		\thanks{Shiva Raj Pokhrel and Gang Li  are with 
			the School of IT, Deakin University, Geelong, 	VIC 3125, Australia~(email: shiva.pokhrel@deakin.edu.au;gang.li@deakin.edu.au).}
		\thanks{Athanasios V. Vasilakos is with Center for AI Research (CAIR), University of Agder (UiA), Grimstad, Norway~(th.vasilakos@gmail.com).}
		%	 the Department of ICT, Center for AI Research(CAIR), University of Agder, Grimstad, Norway~(email: thanos.vasilakos@uia.no).}%
	\thanks{Tianqing Zhu is with the Faculty of Data Science, City University of Macau, Macau, China~(email: Tqzhu@cityu.edu.mo).}}%

%
%\author{\IEEEauthorblockN{Baobao Song}
	%\IEEEauthorblockA{\textit{Faculty of Engineering and IT} \\
		%\textit{University of Technology Sydney, Melbourne, Australia}\\
		%baobao.song@student.uts.edu.au}
	%\and
	%\IEEEauthorblockN{Shiva Raj Pokhrel}
	%\IEEEauthorblockA{\textit{School of Information Technology} \\
		%	\textit{Deakin University, Geelong, Australia}\\
		%	shiva.pokhrel@deakin.edu.au}
	%\and
	%\IEEEauthorblockN{Athanasios V. Vasilakos}
	%\IEEEauthorblockA{\textit{Department of ICT} \\
		%	\textit{University of Agder, Grimstad, Norway}\\
		%	thanos.vasilakos@uia.no}
	%\and
	%\IEEEauthorblockN{Tianqing Zhu}
	%\IEEEauthorblockA{\textit{Faculty of Data Science} \\
		%	\textit{City University of Macau, Macau, China}\\
		%	Tqzhu@cityu.edu.mo}
	%\and
	%\IEEEauthorblockN{Gang Li}
	%\IEEEauthorblockA{\textit{School of Information Technology} \\
		%\textit{Deakin University, Geelong, Australia}\\
		%gang.li@deakin.edu.au}
	%}

\maketitle

\input{./tex/abstract}


%\begin{IEEEkeywords}
%component, formatting, style, styling, insert
%\end{IEEEkeywords}

%=================================================================

\input{./tex/intro}
\input{./tex/preliminaries}
\input{./tex/method}
\input{./tex/experiment}

\input{./tex/conclusions}
\bibliography{yourbib}
\bibliographystyle{IEEEtranN}

\end{document}

%% file: preamble.tex
\ifnum\DraftStatus=1
	\usepackage[draft,colorinlistoftodos,color=orange!30]{todonotes}
\else
	\usepackage[disable,colorinlistoftodos,color=blue!30]{todonotes}
\fi 
\makeatletter
 \providecommand\@dotsep{5}
 \def\listtodoname{List of Todos}
 \def\listoftodos{\@starttoc{tdo}\listtodoname}
 \makeatother
\usepackage{color}

\usepackage{graphicx}
\graphicspath{{./figures/}{./graphics/}{./graphics/logos/}}
\usepackage{array}
\usepackage{booktabs}
\usepackage{makecell}
\usepackage{tabularx}
\usepackage{algorithm}
\usepackage{algorithmic}
\usepackage{breqn}
\usepackage{subcaption}
\usepackage{multirow}
\usepackage{psfrag}
\usepackage{url}
\usepackage[colorlinks,citecolor=blue]{hyperref}
\usepackage{cleveref}
\usepackage{booktabs}
\usepackage{rotating}
\usepackage{colortbl}
\usepackage{paralist}
\usepackage{epstopdf}
\usepackage{nag}
\usepackage{microtype}
\usepackage{siunitx}
\usepackage{nicefrac}
 \usepackage{grffile}
 \usepackage{braket}
\usepackage{fontawesome}
\usepackage{xcolor}
\usepackage{multicol}
\usepackage{wrapfig}
\usepackage{todonotes}
\usepackage{tablefootnote}
\usepackage{threeparttable}
\usepackage{bbm} 
\usepackage{cite}
\usepackage{lipsum}
\usepackage[english]{babel}
\usepackage[pangram]{blindtext}
\usepackage{tikz}
\usetikzlibrary{fit,positioning,arrows.meta,shapes,arrows}
\begin{TulipStyle}
\usepackage[numbers]{natbib}
\usepackage{gitinfo2}
\usepackage{mathtools}
\usepackage{amsthm}
\theoremstyle{definition}

\theoremstyle{remark}

\numberwithin{equation}{section}
\newtheorem{proposition}{Proposition}
\newtheorem{theorem}{Theorem}
\newtheorem{lemma}{Lemma}

\newtheorem{definition}{Definition}

\hypersetup
{
    pdfauthor={\gitAuthorName},
    pdfsubject={TULIP Lab},
    pdftitle={},
    pdfkeywords={TULIP Lab, Data Science},
}

\end{TulipStyle}

%% file: tex/abstract.tex
\begin{abstract}
% Quantum machine learning (QML) promises major computational advantages, 
% yet models trained on sensitive data may inadvertently memorize individuals, 
% creating privacy risks. 
% \emph{Quantum differential privacy} (QDP) mechanisms provide worst-case theoretical guarantees, 
% but they lack empirical auditing tools to verify privacy leakage in deployed QML models.
% In this paper, 
% we present the first black-box privacy auditing framework 
% for QML models based on \emph{Lifted Quantum Differential Privacy} (Lifted QDP). 
% The framework leverages quantum canaries, 
% carefully offset-encoded quantum states, 
% to detect memorization and empirically estimate the privacy leakage incurred during training.
% We not only link the canary offset to a trace distance bound theoretically, 
% but also empirically derive a rigorous lower bound on the privacy budget consumed during training, 
% thereby bridging the gap between theoretical guarantees and practical privacy auditing in quantum.
% Extensive evaluations across simulated and real quantum hardware 
% demonstrate the practicality and effectiveness of our framework
% in estimating real-world privacy loss in QML models.
Quantum machine learning (QML) promises significant computational advantages, yet models trained on sensitive data risk memorizing individual records, creating serious privacy vulnerabilities. While \emph{quantum differential privacy} (QDP) mechanisms provide theoretical worst-case guarantees, they critically lack empirical verification tools for deployed models. We introduce the first black-box privacy auditing framework for QML based on \emph{Lifted Quantum Differential Privacy}, leveraging quantum canaries—strategically offset-encoded quantum states—to detect memorization and precisely quantify privacy leakage during training. Our framework establishes a rigorous mathematical connection between canary offset and trace distance bounds, deriving empirical lower bounds on privacy budget consumption that bridge the critical gap between theoretical guarantees and practical privacy verification. Comprehensive evaluations across both simulated and physical quantum hardware demonstrate our framework's effectiveness in measuring actual privacy loss in QML models, enabling robust privacy verification in QML systems.
\end{abstract}

%% file: tex/intro.tex
%=================================================================
\section{Introduction}\label{sec-intro}

Quantum computing leverages superposition and entanglement 
to achieve computational advantages over classical systems~\cite{nielsen2001quantum}. 
Quantum machine learning (QML) exploits these quantum phenomena 
to potentially accelerate optimization and enable processing 
of classically intractable problems~\cite{khurana2024quantum}. 
However, the enhanced parallel processing capabilities of 
QML algorithms~\cite{lopez2024quantum} amplify privacy risks beyond those of classical counterparts. 
The inherent opacity of quantum systems further complicates privacy analysis 
and leakage detection~\cite{guan2023detecting}. 
Classical privacy-preserving techniques fail in quantum settings. 
Differential privacy (DP) requires repeatable data access, 
but quantum measurement irreversibly collapses quantum states, 
precluding direct application of classical DP. Additionally, 
quantum adversaries compromise classical cryptographic protections~\cite{tom2023quantum}.

Quantum Differential Privacy (QDP) addresses these limitations 
by ensuring quantum outputs remain $\epsilon$-indistinguishable 
for neighboring quantum states differing by small quantum distance~\cite{zhou2017differential}. 
Formally, a quantum mechanism $\mathcal{M}$ satisfies $\epsilon$-QDP 
if for neighboring states $\rho, \sigma$ with $\|\rho - \sigma\|_1 \leq d$:
\begin{equation}
\mathcal{D}(\mathcal{M}(\rho) \| \mathcal{M}(\sigma)) \leq \epsilon
%D_{\infty}\!\big(\mathcal{M}(\rho)\,\big\|\,\mathcal{M}(\sigma)\big) \le \epsilon
\end{equation}
where $\mathcal{D}$ denotes quantum max-relative entropy 
and $\epsilon$ is the  privacy loss parameter .
%where $D_{\infty}$ denotes max divergence (or max–relative entropy).
QDP exploits intrinsic quantum noise (decoherence, measurement uncertainty) 
as a native privacy resource, 
eliminating external noise 
injection~\cite{hirche2023quantum}. 
Existing QDP research focuses on mechanism design~\cite{hirche2023quantum, bai2024quantum, zhou2017differential}, 
while privacy auditing for deployed QML models remains unaddressed. 
The sole auditing technique~\cite{guan2023detecting} requires white-box access 
and provides loose worst-case bounds on $\epsilon$, limiting practical applicability.

\subsection{Motivation}
Current QML privacy auditing lacks black-box techniques 
suitable for realistic deployments. We require an auditing framework satisfying:
\begin{itemize}
\item \textit{Black-box operation}: No access to internal circuits, gradients, or training data
\item \textit{Statistical estimation}: Quantitative privacy loss bounds rather than binary classification
\item \textit{Computational efficiency}: Minimal overhead for regulatory compliance
\end{itemize}

We propose the first black-box auditing framework for QML models 
based on Lifted QDP~\cite{10993414}. 
Our approach introduces \textbf{quantum canaries}---offset-encoded quantum states 
%$|\psi_{\text{offset}}\rangle = e^{i\theta}|\psi\rangle$ 
$|\psi_{\text{offset}}\rangle =e^{-i\frac{\theta}{2}\,\sigma_a}|\psi\rangle$ 
designed to detect memorization through controlled perturbations with magnitude $\theta$, 
where $\sigma_a\in\{X,Y,Z\}$ denotes the Pauli matrix along axis $a\in\{x,y,z\}$. 
The trace distance between original and offset states satisfies:
\begin{equation}
\|\rho - \rho_{\text{offset}}\|_1 = 2|\sin(\theta/2)|
\end{equation}
for pure states. By analyzing the responses of the model to quantum canaries, 
we estimate the lower bounds on the effective privacy budget, 
%$\epsilon_{\text{eff}}$, 
providing interpretable privacy leakage measures under black-box constraints.

\subsection{Contributions}
Our key contributions are as follows.
\begin{itemize}
    \item First black-box auditing framework for QML models 
    providing statistically grounded lower bounds on privacy budget via Lifted QDP;
    \item Introduction of quantum canaries with theoretical guarantees
     relating offset magnitude $\theta$ to trace distance; and
    \item Experimental validation across datasets 
    and quantum circuits, including simulations and real quantum hardware evaluation
\end{itemize}

Our paper is organized as follows.
\Cref{sec-preliminaries} introduces the necessary background and reviews related work.
\Cref{sec-methods} presents the proposed Lifted QDP for Auditing QML framework, 
including the construction of quantum canaries 
and the design of the auditing algorithm.
\Cref{sec-qmldp} demonstrates the effectiveness and efficiency of the framework
across different datasets
and analyzes the impact of various parameter settings.
Finally, conclusions and future directions are provided in \Cref{sec-conclusions}.

%% file: tex/preliminaries.tex
% !TeX spellcheck = <none>
%=================================================================
\section{Preliminaries and Related work}\label{sec-preliminaries}
This section provides the background of QML, QDP, and privacy auditing.
First, it introduces the foundational components of quantum algorithms, 
data encoding methods used to map classical data into quantum states, 
and representative quantum machine learning (QML) model structures.
\Cref{sec:qdp} presents the concept of QDP and Lifted QDP, and their typical mechanisms for realization. 
\Cref{relatedwork} reviews prior work on model auditing and privacy analysis,
especially for quantum setting.

% \subsection{Quantum Machine Learning}\label{sec:qml}
% \subsubsection{Quantum Algorithm}
Quantum algorithm involves three fundamental stages: 
initializing an \emph{input quantum state}, 
evolving the state through unitary operations implemented via \emph{quantum circuits}, 
and performing \emph{quantum measurements} to obtain classical outputs. 
In the following, we outline each of stage in turn.

\textbf{Input Quantum State}
The input to a quantum algorithm is typically qubits, 
the fundamental units of quantum information.
A qubit $|\psi\rangle$ can exist in a \textit{superposition} of
the computational basis states $|0 \rangle$ and $|1 \rangle$,i.e.,
\begin{equation}\label{equ:1}
	|\psi\rangle = \alpha |0\rangle + \beta |1\rangle, \quad \text{where } |\alpha|^2 + |\beta|^2 = 1.
\end{equation}
Any pure single-qubit state can be represented 
as a point on the surface of the \textit{Bloch sphere}, 
a unit sphere in three-dimensional space. 
Extending to $n$ qubits,
we take tensor products of the individual Hilbert spaces,
therefore forming a $2^n$-dimensional Hilbert space. 

\textbf{Quantum Circuits}
A quantum algorithm is typically realized as a quantum circuit,
which applies a sequence of quantum gates to an initial state 
to carry out the desired computation.
A quantum gate is a unitary operator $U$ acting
on one or more qubits.
These operations can change the states of qubits, 
including flipping, rotating, creating superposition.
Common single-qubit gates include the Pauli gates
$X$, $Y$, and $Z$, the Hadamard gate $H$,
and the rotation gates $R_x(\theta)$, $\$R_y(\theta)$ and $R_z(\theta)$,
which perform rotations about the respective axes of the Bloch sphere. 
Multi-qubit gates such as the CNOT gate and controlled-$Z$ gate, 
operate on two or more qubits and are indispensable for generating entanglement.
A quantum circuit of depth $d$ is an ordered sequence of such gates,
\begin{equation}\label{equ:2}
\mathcal{E}=U_{d}\,\cdots\,U_{2}\,U_{1}.
\end{equation}
A quantum circuit transforms an initial input state 
into a final quantum state through $n$ gates. 
In quantum machine learning, 
parameterized quantum circuits are commonly used, 
where certain gates include tunable parameters optimized during training. 

\textbf{Quantum Measurement}
Quantum measurement is the final step in a quantum algorithm, 
converting a quantum state into classical information 
through probabilistic projection onto a measurement basis.
This process is modelled by a positive-operator-valued measure (POVM)~\cite{oszmaniec2017simulating},
represented by a set of measurement operators $\{M_i\}_{i \in \mathcal{O}}$ 
acting on the Hilbert space,
and satisfying the completeness relation $\sum_i M_i = I$.
Due to finite sampling (i.e., a limited number of repeated circuit executions), 
the observed outcomes are subject to shot noise,
a source of statistical uncertainty.
\subsubsection{QML Data Encoding}
Quantum data encoding, 
also known as quantum feature mapping, 
refers to the process of embedding classical data 
into quantum states that can be processed by quantum circuits. 
Different encoding strategies induces a specific geometry in Hilbert space 
and affecting the expressive power of the resulting quantum model.

A common encoding method is \texttt{angle encoding}, 
where each feature is mapped to a rotation gate.
These gates typically perform rotations around the $x$-, 
$y$-, or $z$-axis of the Bloch sphere, 
using gates such as $R_x(x_i)$, $R_y(x_i)$, or $R_z(x_i)$.
For an input vector $x = (x_1, x_2, \dots, x_d)$, 
angle encoding $R_y$ applies one rotation per feature to an initially prepared $|0\rangle^{\otimes d}$state:
\begin{equation}\label{equ:3}
|\vec{x}\rangle =\bigotimes_{i=1}^{n} R_y(x_i) |0\rangle.
\end{equation}

This encoding maps each component $x_i$ to a quantum rotation
and produces a feature-dependent state geometry in Hilbert space.
One of the key advantages of angle encoding is its hardware efficiency. 
It yields encoding layer uses only single-qubit rotations
and requires only a linear number of qubits relative to the feature size. 

\subsubsection{QML Model Structures}

Quantum machine learning (QML) models can generally be classified into three categories:  
quantum versions of classical algorithms 
(replicate classical models using quantum circuits),
quantum-inspired machine learning 
(draw on quantum principles to improve classical algorithms)
and hybrid classical-quantum models (quantum circuits are integrated into classical machine learning, shown in \Cref{fig:qml})~\cite{zeguendry2023quantum}.
Here, 
we focus on the most widely used category, hybrid QML models, 
due to their flexibility, trainability, 
and compatibility with nowadays quantum hardware.
In the paper,
all QML models in below refer specifically to this hybrid classical–quantum architecture.

\begin{figure}[h]
	\centering
\includegraphics[width=0.95\linewidth]{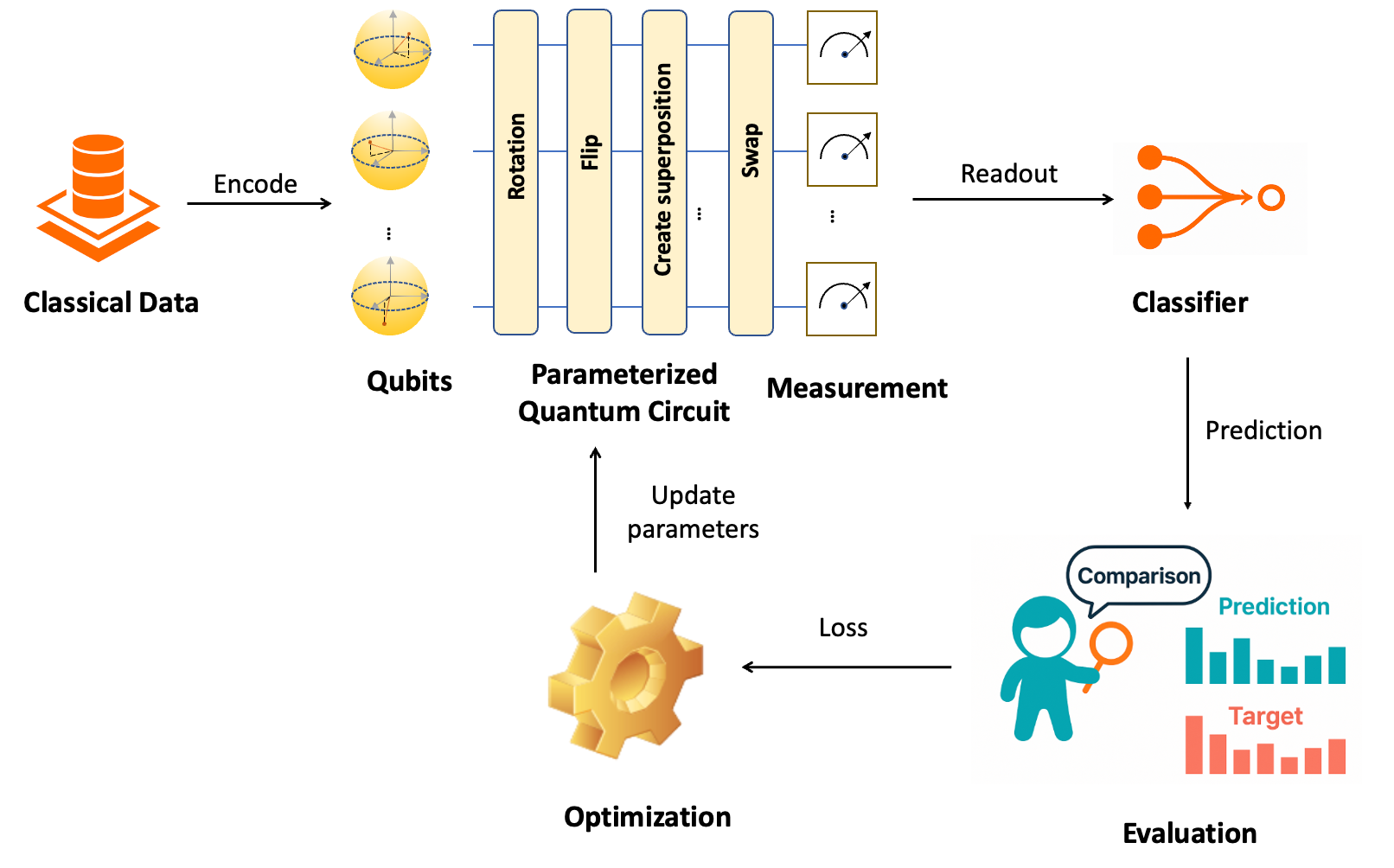}
	\caption{The framework of hybrid classical-quantum models}
	\label{fig:qml}
\end{figure}

QML models combine quantum circuit, 
typically parameterized quantum circuits (PQCs),
with classical components such as optimizers. 
As shown in \Cref{fig:qml},
classical data is first encoded into quantum states, 
which are then processed through PQCs to perform transformations.
The resulting quantum states are measured 
and read out into classical information, 
which is passed to a classical classifier to generate predictions. 
These predictions are compared with ground truth labels in an evaluation step to compute a loss, 
which is subsequently used by a classical optimization algorithm to update the parameters of the quantum circuit. 
This iterative process defines a class of models often referred to as \emph{Hybrid Quantum Neural Networks}~\cite{adhikary2020supervised}  (QNNs), 
where the PQC plays a role analogous to a neural network layer in classical deep learning.
QNNs leverage both quantum feature encoding and variational transformations, 
trained via classical feedback, 
to model complex relationships in data.
This hybrid framework is particularly suitable for noisy intermediate-scale quantum (NISQ) devices, 
as it minimizes quantum resource requirements 
and leverages the strengths of classical computing.

\subsection{QDP and Lifted QDP}\label{sec:qdp}
\emph{Quantum Qifferential privacy} (QDP)~\cite{zhou2017differential} extends the protective principle 
of classical Differential Privacy (DP)~\cite{dwork2006differential} to quantum-computing settings.
It offers a formal guarantee that, 
for any pair of neighboring quantum input states, 
the output states produced by a quantum algorithm 
(e.g.  set of instructions that uses quantum bits to process information) 
are statistically indistinguishable under every possible measurement, 
ensuring that the addition or removal of a sensitive quantum state
does not significantly affect the statistics observable to an adversary.
Here, 
two quantum states are considered neighbors
if they are close to each other in terms of a chosen distance metric, 
such as the trace distance~\cite{hirche2023quantum},
which is the sum of the absolute values of the differences
between the eigenvalues of the two states.
\begin{definition} \label{def:td}	\textbf{Neighboring Quantum States}. 
Given two quantum states $\rho$ and $\sigma$,
they are considered neighboring 
if their trace distance is below a certain threshold $d$:
\begin{equation}\label{equ:9}
	\tau(\rho,\sigma)=\frac{1}{2} \operatorname{Tr}(|\rho-\sigma)|) =\frac{1}{2} \sum_i |\lambda_i(\rho) - \lambda_i(\sigma)|\leq d.
\end{equation}
where trace $\operatorname{Tr}(\cdot)$ denotes the matrix trace,
and $\lambda_i(\rho)$ and $\lambda_i(\sigma)$ 
denote the eigenvalues of $\rho$ and $\sigma$ respectively.
\end{definition}
Based on the above definition of neighboring quantum states,
the formal definition of QDP is given below.
\begin{definition} \label{def:qdp}	\textbf{Quantum Differential Privacy}. 
    Consider a quantum algorithm $\mathcal{A}=\{\mathcal{E}, \{M_i\}_{i \in \mathcal{O}}\}$ 
    where $\mathcal{E}$ is a quantum circuit and
	$\{M_i\}_{i \in \mathcal{O}}$ is a set of measurement operators forming a POVM with outcome set $\mathcal{O}$.
    The quantum algorithm $\mathcal{A}$ satisfies $(\epsilon, \delta)$-QDP 
	if for any pair of neighboring quantum states $\rho$ and $\sigma$, 
	and any measurement output $\mathcal{O}$, 
	it holds that:
	\begin{equation}\label{equ:10}
		\sum_{i \in \mathcal{O}} \operatorname{Tr}[M_i \mathcal{E}(\rho)] \leq e^{\epsilon} \sum_{i \in \mathcal{O}} \operatorname{Tr}[ M_i \mathcal{E}(\sigma)] + \delta
	\end{equation}
	where $\epsilon$ quantifies the privacy loss 
    and $\delta$ accounts for a small probability of failure in the privacy guarantee. 
    When $\delta = 0$, 
    the algorithm satisfies pure $\epsilon$-QDP, meaning that the privacy guarantee holds with absolute certainty.
\end{definition}

QDP can leverage the intrinsic noise 
from quantum channels or quantum measurements to achieve privacy protection~\cite{10993414}.
Typical channel noises include depolarizing noise~\cite{du2021quantum}, 
which drives a qubit towards the maximally mixed state, 
and generalised amplitude damping (GAD) noise~\cite{srikanth2008squeezed}, 
which causes qubits to lose their excitation 
and drift toward a thermal equilibrium state.
Meanwhile,
shot noise in the quantum measurement stage
can increase output uncertainty and enhancing privacy.
The QDP mechanism for depolarizing noise, 
GAD noise, and shot noise are described as follows:

\textbf{QDP Mechanism of Depolarizing Noise}~\cite{zhou2017differential}. \label{the:dep} 
Given the trace distance $d$ of two neighboring qubits
and the dimension of the Hilbert space $D$, 
the depolarizing noise channel provides $\epsilon$-QDP with
\begin{equation}\label{equ:12}
\epsilon_{dep} = \ln \left[ 1 + \frac{(1 - p)dD}{p} \right]
\end{equation}
where $p$ denotes the depolarizing probability.
% \end{theorem}

% \textbf{QDP Mechanism of GAD Noise}~\cite{zhou2017differential}. \label{the:gad} 
% Given the trace distance $d$ of two neighboring qubits
% and the dimension of the Hilbert space $D$,
% the GAD noise channel provides $\epsilon$-QDP with
% \begin{equation}\label{equ:16}
% \epsilon_{gad} = \ln \left[ 1 + \frac{2d\sqrt{1-\gamma}}{1-\sqrt{1-\gamma}} \right]
% 	\end{equation}
% where $\gamma$ is the the energy exchange possibility
% between the system and the environment.
% \end{theorem}

% \begin{theorem}
\textbf{QDP Mechanism of Measurement Noise}~\cite{li2023differential}. \label{the:mea} 
Assume the number of measurements is $N$
and the set of measurement projection operators is $\{M_i\}$ 
with maximum rank $r$ of $\{M_m\}$. 
For quantum states $\rho$ and $\sigma$ with $F(\rho, \sigma) \leq d$, 
the measurement process holds ($\epsilon$, $\delta$)-QDP with
\begin{equation}\label{equ:20}
\epsilon = 
\frac{Ndr}{\mu(1 - \mu)} 
\left[
\frac{(1 - 2\mu - Ndr)c^2}{2\mu(1 - \mu - Ndr)} 
+ c 
+ \frac{Ndr}{2}
\right]
\end{equation}
where $\mu = \min\{\mu_1, \mu_0\}$ 
($\mu_0 = \operatorname{Tr}(\rho M_m)$ and
$\mu_1 = \operatorname{Tr}(\sigma M_m)$), 
$\delta = \sqrt{2\pi} \sigma \, \mathrm{erfc} \left( \frac{c}{\sqrt{2} \sigma} \right)$,
and $\sigma = \sqrt{\frac{\mu (1 - \mu)}{N}}$.
% \end{theorem}

As QDP is defined with respect to one fixed pair of neighbouring states, 
every privacy audit must train a new model on that specific pair and run a separate hypothesis test. 
Repeating this process for different pairs quickly becomes computationally expensive.
Therefore, 
Lifted QDP~\cite{10993414} is proposed to replaces that fixed setup with a randomised neighbouring-state pairs 
and rejection events drawn from a joint distribution. 
Randomisation lets many correlated hypothesis tests be re-used on the same trained model, 
reducing sample complexity and computation, 
while rejection sets give higher detection power against privacy leaks.
The definition of Lifted QDP is: 

\begin{definition} \label{def:liqdp} \textbf{Lifted Quantum Differential Privacy}.  
Let $P$ represent a joint probability distribution over ($\rho$, $\sigma$, $R$), 
where ($\rho$, $\sigma$) is a pair of random neighboring quantum states and $R$ is a rejection set. 
A quantum algorithm $\mathcal{A}=\{\mathcal{E}, \{M_i\}_{i \in \mathcal{O}}\}$ satisfies $(\epsilon, \delta)$-Lifted QDP if, 
for any quantum measurement $M$, 
the following condition holds:
\begin{equation}\label{equ:10}
    \mathbb{P}_{\mathcal{A},\mathcal{P}}[\sum_{i \in \mathcal{O}} \operatorname{Tr}[M_i \mathcal{E}(\rho)] \in R] \leq e^{\epsilon} \cdot \mathbb{P}_{\mathcal{A},\mathcal{P}}[\sum_{i \in \mathcal{O}} \operatorname{Tr}[ M_i \mathcal{E}(\sigma)] \in R] + \delta.
\end{equation}
\end{definition}

Lifted QDP turns the single-pair, single-test guarantee of standard QDP into 
an guarantee over an entire distribution of canaries and tests.
Because the auditor can choose
a single trained model can be tested many times 
with different, but statistically correlated, 
quantum canary pairs ($\rho$, $\sigma$) and rejection sets $R$,
and the resulting collection of hypothesis tests can be combined to estimate the privacy leakage. 
Besides,
Lifted QDP can be viewed as a strict generalisation of QDP.
When the lifted joint distribution $\mathcal{P}$
is fixed independently of the algorithm $\mathcal{A}$, 
Lifted QDP is equivalent to QDP~\cite{10993414}.
% QDP inherits the post-processing property from classical DP. 
% That is, if a quantum algorithm (or mechanism) $\mathcal{A}$ satisfies $(\epsilon, \delta)$-QDP, 
% then for any channel/function $\mathcal{F}$, 
% the composed algorithm $\mathcal{F} \circ \mathcal{A}$ 
% also satisfies $(\epsilon, \delta)$-QDP. 

\subsection{Related Work on Auditing}\label{relatedwork}

In classical ML, 
privacy auditing is a critical technique 
to empirically assess whether a model conforms to its claimed DP guarantees.
A common method is the canary-based audit in black-box setting~\cite{namatevs2025privacy}, 
in which specially crafted examples (called canaries) 
are injected into the training dataset, 
and hypothesis testing is performed on the model’s output 
to determine if those canaries have been memorized.
Early audits required training hundreds of models 
because each hypothesis test compared two separately-trained networks 
that differed by a single canary~\cite{carlini2019secret}. 
Recent work has significantly reduced this cost 
by allowing multiple canaries to be tested 
in parallel within a single training run~\cite{pillutla2024unleashing,steinke2023privacy,jagielski2020auditing}.
Jagielski et al.~\cite{jagielski2020auditing} introduced a deterministic approach 
by directly injecting multiple canaries into the dataset 
and conducting multiple hypothesis tests to detect vulnerabilities in DP-SGD. 
While effective, this method has limitations due to its lack of randomness and diversity, 
which means it cannot fully address all potential privacy issues.
Therefore, 
Steinke et al.~\cite{steinke2023privacy} adopted a randomized approach, 
where canaries are sampled from a Poisson distribution. 
This method leverages randomness to ensure 
that canaries are included or excluded independently, 
which eliminates group privacy concerns, 
and allows for the reuse of models across multiple tests.
Pillutla et al. \cite{pillutla2024unleashing} introduced Lifted DP, 
which expands the DP definition to handle randomized datasets.
By incorporating multiple random canaries, 
their framework allows for several tests to be run simultaneously using one model. 
They demonstrated that this framework makes the auditing process significantly more sample-efficient, 
without compromising the privacy guarantees.
Recent advancements in canary-based auditing 
have also extended to large language models~\cite{panda2025privacy,meeus2025canary},
particularly in the context of privacy risks associated with synthetic data generation, 
where new canary designs have been introduced to enhance the effectiveness of privacy audits, 
improving both memorization detection and the influence on synthetic data outputs.

However, 
these classical auditing approaches cannot be directly extended to QML
due to fundamental differences between classical and quantum computations. 
Quantum algorithms operate on quantum states 
and leverage phenomena such as entanglement and superposition, 
making their behavior and the impact of privacy mechanisms 
distinct from classical counterparts. 
% Consequently, dedicated auditing frameworks for QDP need to be developed.
Up to now,
Guan et al. \cite{guan2023detecting} is the only one work designed forverifying QDP violations.
Unlike classical canary-based black-box auditing, 
their method employs a white-box approach 
that directly analyzes the internal quantum operations of algorithms.
They systematically examines quantum circuits 
by analyzing output state distributions 
through linear algebraic techniques, 
formulating QDP verification as a mathematical problem involving eigenvalues 
and eigenvectors of certain matrices derived from quantum operations.
The privacy leakage estimated by their method represents an upper bound on $\epsilon$, 
often approaching theoretical values; however, 
these estimates are typically excessively high in practical scenarios.
Furthermore, 
the white-box setting of their method 
limits its applicability 
in contemporary quantum computing environments, 
where internal algorithm details may be inaccessible or protected. 
Consequently, 
there is a need for a practical auditing method 
based on statistical hypothesis testing to more estimate actual privacy leakage in QML models.

%% file: tex/method.tex
\section{Lifted QDP for Auditing QML Framework}\label{sec-methods}

Based on \texttt{Lifted QDP},
we propose a novel framework to audit QML models effectively.
By constructing multiple quantum canaries, 
this framework enables the tracing of model behavior under subtle, information-carrying perturbations, 
thereby facilitating the assessment of the model’s privacy protection capabilities.
\Cref{sec-pro} formally defines the privacy auditing problem in QML settings, 
outlining the threat model and auditing objectives
The proposed framework is elaborated in \Cref{sec-framework}, 
which introduces its core components 
and the underlying auditing intuition. \Cref{sec-qcanaries} details the construction of quantum canaries 
via tailored encoding strategies. \Cref{sec-algorithm} present the \texttt{Lifted QDP-based Auditing Algorithm} for practical deployment.

\subsection{Problem Statement}
\label{sec-pro}

Let $\theta$ denote a QML model trained on a private quantum dataset $\mathcal{D}$ 
using a parameterized quantum circuit. 
The auditor does not have access to the internal structure of $\theta$, 
including its circuit design, Kraus operators, or observables, 
but is allowed to (i) submit data for training via a training API, 
and (ii) query the model's predictions for inputs through an inference interface.
The adversary (e.g., the model trainer) is honest-but-unreliable: 
privacy violations may occur due to misconfigured or insufficiently randomized training procedures, 
rather than malicious intent. 
Our goal is to estimate the extent to which the model $\theta$ has memorized specific training inputs, 
a phenomenon indicative of privacy leakage($\epsilon$), 
while minimizing the required sample size.
Here, the sample size refers to the number of independent model training runs needed to perform statistical auditing. 
Therefore, given the $n$ independent calls to a training oracle $\texttt{Train}$, 
a trace distance bound $d$ defining the closeness of neighboring quantum states, 
and a failure probability $\beta$,
the auditing model is designed as
\begin{equation}\label{equ:def}
\mathcal{F}\colon\bigl(\textbf{Train},\textbf{Infer},d,n,\beta\bigr)\;\longmapsto \hat{\epsilon}.
\end{equation}

\subsection{Auditing Idea and Framework}
\label{sec-framework}

Auditing QML models aims to detect whether 
such models memorize or inadvertently leak information about quantum data. 
Inspired by classical ML auditing method, 
we adapt the notion of canaries to the quantum setting 
by designing \textit{quantum canaries}, 
quantum-encoded inputs that carry carefully embedded sensitive information. 
These canaries are constructed to interact with the model in a subtle way, 
enabling us to probe its internal behavior 
and infer potential privacy breaches. 

The core idea behind auditing QML models is intuitive: 
if a model memorizes specific training inputs, 
it is likely to exhibit distinguishable behavior 
when exposed to those same inputs again. 
In our approach, some quantum canaries are inserted into the training data 
and then reintroduced during testing. 
By comparing the model's responses to these trained quantum canaries 
with those to untrained, unseen quantum canaries, 
we can infer whether the model has retained information 
specific to the original quantum canaries.
A central challenge of this strategy lies in 
constructing a rigorous and efficient framework 
that enables quantifiable detection of 
such leakage under quantum constraints.

Therefore, we build our auditing framework upon the concept of Lifted QDP. 
While classical QDP assesses privacy by bounding the distinguishability
between the outputs of a quantum mechanism on neighboring quantum states, 
it inherently considers a fixed auditing setup, 
typically, a single pair of neighboring states and a specific test condition.
This formulation hinders flexibility in the design of auditing mechanisms 
and leads to high sample complexity, 
as each privacy test typically 
requires a fresh model training run and a dedicated hypothesis test.
Lifted QDP overcomes these limitations by
considering a randomized ensemble of neighboring quantum state pairs 
and rejection sets drawn from a joint distribution.
This expanded formulation enables the use of randomized quantum canaries (from the certain distribution)
and supports multiple correlated hypothesis tests per model instance. 
By leveraging the test statistics  induced by quantum canaries, 
Lifted QDP allows statistical reuse of each trained model, 
significantly reducing the number of required training runs. 
Moreover, the use of per quantum canary rejection sets (e.g., a loss-threshold rule learned on separate data) 
and the adaptive confidence intervals that exploit empirical correlations, 
boosts auditing power while preserving rigorous finite-sample guarantees. 
This probabilistic lifting not only retains the core privacy guarantees of standard QDP 
but also provides a more expressive and sample-efficient foundation 
for detecting subtle privacy violations in QML models. 
Building upon this definition, 
we present a complete auditing framework for QML models,
which is shown as \Cref{fig:auditing}.

\begin{figure}[h]
	\centering
\includegraphics[width=0.95\linewidth]{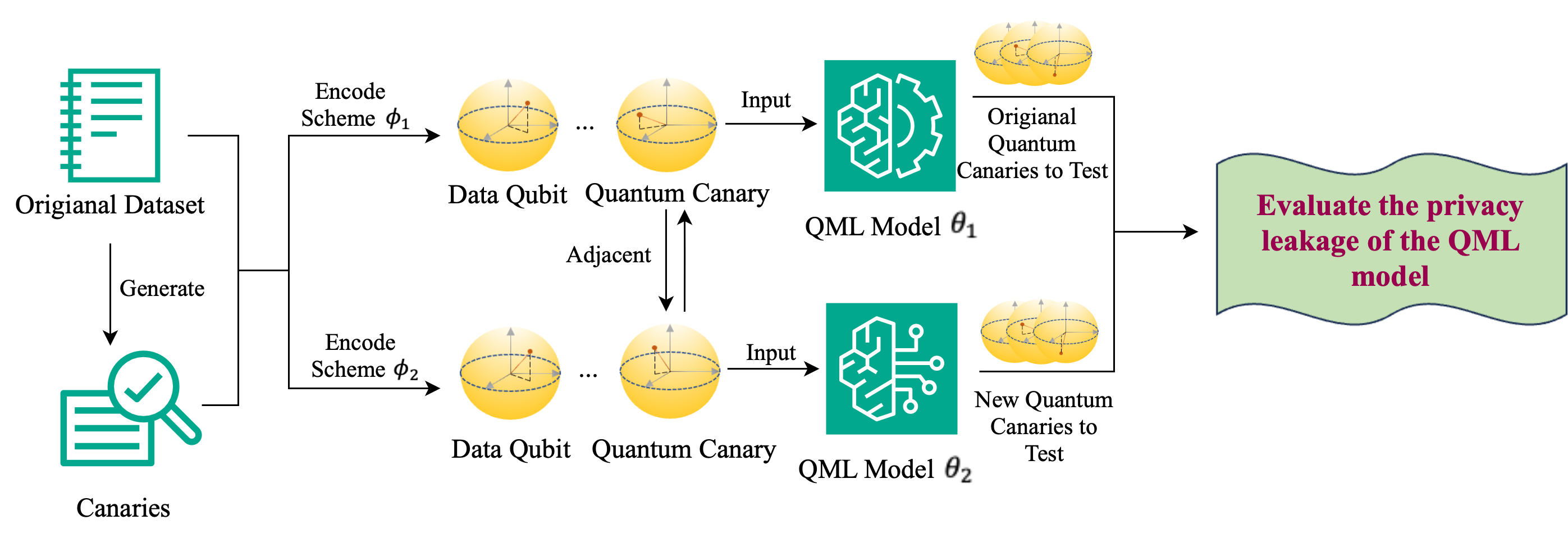}
	\caption{The auditing framework for QML models}
	\label{fig:auditing}
\end{figure}

In this framework, 
we begin with an original dataset and generate two equally sized sets of 
canaries-specially constructed classical data points 
that are intended to encode sensitive information.
These canaries are specially designed to match the statistical characteristics of the original data distribution.
One set of canaries is transformed into quantum canaries by applying two quantum encoding schemes  $\phi_1$ and $\phi_2$,
ensuring that the resulting quantum states 
satisfy the adjacency condition required by Lifted QDP,
namely, a trace distance below a threshold  $d$.
Both the original dataset and the first set of canaries 
are independently encoded into quantum states.
These states, 
comprising data qubits and quantum canaries, 
are then used as input to train two QML models, 
denoted $\theta_1$ and $\theta_2$.
Only the first set of canaries is included in the training phase, 
while the second set is held out exclusively for evaluation.
This setup enables the simulation of randomized ensembles of neighboring quantum states 
and supports parallel hypothesis testing under Lifted QDP.
% Importantly, 
% the framework supports statistical reuse of the trained models: 
% a single model instance can be evaluated 
% against a distribution of quantum canaries, 
% thereby enabling multiple hypothesis tests without retraining.

During evaluation, 
both models are tested exclusively on quantum canaries 
encoded using the same quantum encoding scheme $\phi_1$ (or $\phi_2$). 
Importantly, model $\theta_1$ is tested on the same set of quantum canaries it was trained on,
while model $\theta_2$ is tested on a disjoint set it has never seen during training.
We then measure the loss associated with each quantum canary. 
The loss function is designed to quantify the model’s sensitivity to variations in these inputs, 
serving as a proxy for how much information the model retains 
about the first set of encoded quantum canaries.
If the loss for QML model $\theta_1$ (on seen quantum canaries) 
is significantly lower than the loss for QML model $\theta_2$
(on unseen quantum canaries), 
this indicates that $\theta_1$ has memorized more information about the quantum states, 
reflecting the presence of quantum memorization 
and greater privacy leakage.
By comparing the losses associated with the quantum canaries 
across the two models, 
we can estimate the effective privacy parameter $\epsilon$ of the QML model,
hereby quantifying the degree of leakage 
and assessing the robustness of the QML model’s privacy protections 
in a statistically principled and sample-efficient manner.

\subsection{Construction of Quantum Canaries}\label{sec-qcanaries}

Traditional privacy auditing techniques in classical ML 
rely on canaries—carefully designed data records inserted into the training dataset 
to probe whether a model memorizes and leaks specific inputs. 
These canaries are typically raw data samples drawn from or engineered in the input space. 
Such canaries are inherently classical in nature and in quantum setting (information is represented and manipulated as quantum states),
we have to understand privacy leakage in terms of the distinguishability between quantum states.
This fundamental shift necessitates the construction of Quantum Canaries, 
which are quantum-encoded representations of data 
designed to test the privacy behavior of quantum machine learning models.

Since most QML models encode data one record at a time, 
we must ensure that every record in the dataset, 
after encoding, satisfies the definition of adjacency in QDP. 
To achieve this, 
we adopt a slightly different encoding method for neighboring records.
Specifically, 
for the original training dataset 
$D$, 
both models adopt the same angle encoding scheme for all records, 
namely $\phi(x) = \pi x$, 
where $x$ denotes the inputs. 
This ensures consistency in the quantum representations 
of non-canary data across the two models.
The distinction lies in the treatment of the inserted canaries,
which is shown in \Cref{fig:encode}. 
In model $\theta_1$, 
the canaries are encoded identically to the original data 
using the same angle function $\phi(c) = \pi c$, 
where $c$ represents the generated canaries. 
In contrast, 
model $\theta_2$ encodes the canaries using a slightly perturbed angle encoding of the form:
$\phi'(c) = \pi*c + \alpha$
where $\alpha$ is a small offset 
sampled from a Gaussian distribution $\mathcal{N}(0, \sigma^2)$,
introduced to induce a bounded change 
in the resulting quantum state, 
while maintaining adjacency under QDP.

\begin{figure}[h]
	\centering
\includegraphics[width=0.7\linewidth]{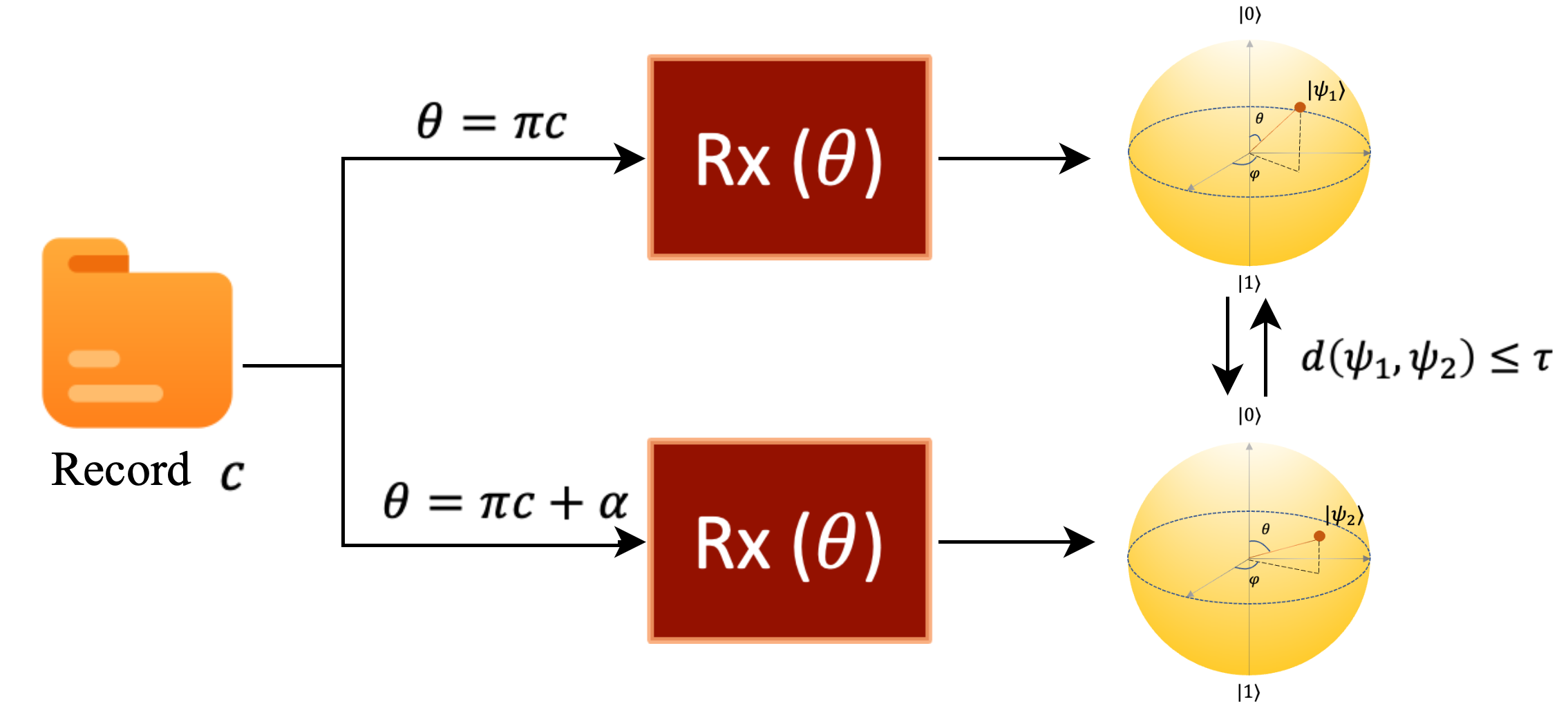}
	\caption{Encode method for constructing quantum canaries}
	\label{fig:encode}
\end{figure}

We formally prove the relationship between the offset angle $\alpha$
and the trace distance upper bound $d$ in the QDP.
The relationship is given as follows:

\begin{theorem}\label{the:the1}
Consider a classical data vector $x \in \mathbb{R}^m$ consisting of $m$ features. Each feature $c_j$ is encoded into a single qubit using either the $R_x$ or $R_y$ rotation gate.  
We define two encoding schemes $\phi(c_j) = \pi c_j$ and $\phi'(c_j) = \pi c_j + \alpha$, 
where $\alpha$ is a offset sampled from a Gaussian distribution $\mathcal{N}(0, \sigma^2)$.
To ensure that the quantum states produced by the two encodings remain sufficiently close (trace distance less than $d$), 
the $\sigma$ is bounded by:
\begin{equation}\label{equ:ec2}
\sigma \leq \frac{2 \arcsin(d)}{2.576}
\end{equation}
\end{theorem}
% Let $\rho$ and $\sigma$ denote the resulting quantum states 
% under the first and second encoding schemes, respectively. 
% The trace distance between $\rho$ and $\sigma$ is bounded by:
% \begin{equation}\label{equ:ec1}
% d(\rho, \sigma) \leq \frac{\alpha}{2}\leq \tau,
% \end{equation}
% \end{theorem}

\begin{proof}
Given two pure states $\ket{\psi_1}$ and $\ket{\psi_2}$, 
which are obtained by encoding classical data using the $R_x$ or $R_y$ rotation gates, 
the trace distance is defined as:
\begin{equation}\label{equ:ec2}
\tau(\ket{\psi_1}\bra{\psi_1}, \ket{\psi_2}\bra{\psi_2}) = \sqrt{1 - |\langle \psi_1 | \psi_2 \rangle|^2}
\end{equation}
Thus, 
the trace distance for a single qubit is:
\begin{equation}\label{equ:ec3}
\tau= \sqrt{ \frac{1 - \cos^2(\phi - \phi')}{2} }=\left| \sin\left( \frac{\phi' - \phi}{2} \right) \right| = \left| \sin\left( \frac{\alpha}{2} \right) \right|< d 
\end{equation}
From this we derive the following:
\begin{equation}\label{equ:ec4}
 \quad |\alpha| < 2 \arcsin(d)
\end{equation}
Now assume that the offset $\alpha$ is sampled from a Gaussian distribution, 
i.e., $\alpha \sim \mathcal{N}(0, \sigma^2)$:
\begin{equation}\label{equ:ec5}
P\left( |\alpha| < 2 \arcsin(d) \right) = 1 - \delta
\end{equation}
Therefore,
\begin{equation}\label{equ:ec6}
\sigma \leq \frac{2 \arcsin(d)}{\Phi^{-1}(1 - \delta/2)}
\end{equation}
where $\Phi^{-1}(\cdot)$ is the inverse cumulative distribution function (inverse CDF) 
of the standard normal distribution
and if we desire $\delta = 0.01$ (i.e., a 99\% probability that the distance is less than $d$, 
then:
\begin{equation}\label{equ:ec7}
\sigma \leq \frac{2 \arcsin(\tau)}{2.576}
\end{equation}
% when $\alpha$ is sufficiently small, we can use the first-order Taylor approximation:
% \begin{equation}\label{equ:ec4}
% \sin\left( \frac{\alpha}{2} \right) \approx \frac{\alpha}{2}
% \end{equation}
% Now assume that the offset $\alpha$ is sampled from a Gaussian distribution, 
% i.e., $\alpha \sim \mathcal{N}(0, \sigma^2)$,
% so $\mathbb{E}[|\alpha|] = \sigma \sqrt{\frac{2}{\pi}}$.
% Then the expected trace distance becomes:
% \begin{equation}\label{equ:ec5}
% \mathbb{E}[d] \approx \frac{1}{2} \mathbb{E}[|\alpha|] = \sigma \cdot \sqrt{\frac{1}{2\pi}}
% \end{equation}
% Since $c_i \in [0,1]$, we have
% \begin{equation}\label{equ:ec4}
% \left|\sin\left( \frac{\alpha c_i}{2} \right)\right| \leq \sin\left( \frac{\alpha}{2} \right) \leq d
% \end{equation}
\end{proof} 

Note that since $\alpha$ is sampled from a Gaussian distribution, 
there is always a non-zero probability of sampling values arbitrarily far from the mean. 
As a result, 
the inequality in \Cref{the:the1} holds with high probability (e.g., $99\%$), 
but not with absolute certainty. 
% This means that, in rare cases, the trace distance between the two quantum states may exceed the threshold $\tau$.
To ensure that the trace distance is strictly bounded by $d$ for all inputs, 
a simple remedy is to clip the sampled value of $\alpha$ to a bounded interval $[-\gamma, \gamma]$.
With this clipping, 
we can obtain the following theorem.

\begin{lemma}\label{the:the2}
Consider a classical data vector $x \in \mathbb{R}^m$ consisting of $m$ features. 
Each feature $c_j$ is encoded into a single qubit using either the $R_x$ or $R_y$ rotation gate.  
We define two encoding schemes: 
$\phi(c_j) = \pi c_j$ and 
$\phi'(c_j) = \pi c_j + \alpha$, 
where $\alpha \sim \mathcal{N}(0, \sigma^2)$ is a Gaussian offset.
To ensure that the quantum states produced by the two encodings remain within a trace distance less than $d$, 
we apply clipping to the offset as follows:
$\alpha_{\text{clipped}} = \text{Clip}(\alpha, -\gamma, \gamma)$, and the clipping bound satisfies:
\begin{equation}\label{equ:ec8}
\gamma \leq 2 \arcsin(d)
\end{equation}
\end{lemma}

To ensure the validity and interpretability of our privacy audit, 
it is essential that the only difference between the two models 
lies in the quantum encoding of the canaries. 
We intentionally avoid applying the encoding scheme $\phi_2$
to the entire training dataset.
This is because encoding all training records in model $\theta_2$ 
using the perturbed encoding scheme 
$\phi'(x) = (\pi + \alpha)x$
will affect the estimation results 
and violate the central assumption 
that the models differ only in the presence and treatment of canaries.
In contract,
by restricting the encoding perturbation to the canaries alone, 
we isolate their contribution to model behavior. 
This design ensures that any observable difference in model output, 
particularly on the canary inputs, 
can be attributed solely 
to the presence of the quantum canaries themselves. 
This is critical for evaluating whether a given canary exerts an undue influence
on the model and for testing whether such influence leads to a privacy violation 
under the Lifted QDP.

%change here
\subsection{Lifted QDP-based Auditing Algorithm}\label{sec-algorithm}
Building upon the auditing framework described in \Cref{sec-framework} 
and the construction of quantum canaries introduced in \Cref{sec-qcanaries}, 
we now present the practical auditing algorithm 
that operationalizes the Lifted QDP-based privacy assessment,
which is shown in \cref{alg1}.

\begin{algorithm}[htbp]
	\renewcommand{\algorithmicrequire}{\textbf{Input:}}
	\renewcommand{\algorithmicensure}{\textbf{Output:}}
	\caption{Lifted QDP for Auditing}
	\label{alg1}
	\begin{algorithmic}[1]
        \REQUIRE Sample size $n$, number of canaries $2K$, QDP mechanism $A$, training set $\mathcal{D}$, threshold $\kappa$, failure probability $\beta$, privacy parameter $\delta$
	\ENSURE $\epsilon$ in QDP
	\FOR{$i \in [n]$}
            \STATE Randomly generate $2K$ canaries {$c_1$, $c_2$, ...,$c_k$,..., $c_{2k}$}
            \STATE $ \mathcal{D}' \leftarrow \mathcal{D} \cup \{c_1, \dots, c_{k}\}$
            \STATE Encode $\mathcal{D}'$ into qubits using angle encoding scheme $\phi_1$ and $\phi_2$ to $\mathcal{D}_0$, $\mathcal{D}_1$ respectively 
            \STATE Initialize QML with quantum circuit
            \STATE
            Train two models $\theta_0 \leftarrow \mathcal{A}(\mathcal{D}'_0)$ and $\theta_1 \leftarrow \mathcal{A}(\mathcal{D}'_1)$
            \STATE  Compute forward pass through QML weights to get prediction and compute loss $l_{z=1}^K=(f_{\theta_1}(c_z))_{z=1}^K$ and $l_{j=K+1}^{2K}=(f_{\theta_0}(c_j))_{k=K+1}^{2K}$
            \STATE Record test statistics $x^{(i)} \leftarrow \mathbb{I} (l_{z=1}^K< \kappa)$ and $y^{(i)} \leftarrow \mathbb{I}(l_{j=k+1}^{2K} < \kappa)$
		\ENDFOR
		\STATE $\quad \text{Set } p_1 \leftarrow \text{XBernLower}( \{x^{(i)}\}_{i \in [n]}, \frac{\beta}{2}) \text{ and } \bar{p_0} \leftarrow \text{XBernUpper}( \{y^{(i)}\}_{i \in [n]}, \frac{\beta}{2})$
	    \STATE $\quad  \hat{\varepsilon}_n \leftarrow \log\left(\frac{p_1 - \delta}{\bar{p_0}}\right) \text{ and a guarantee that } \mathbb{P} (\varepsilon < \hat{\varepsilon}_n) \leq \beta$
	\end{algorithmic}  
\end{algorithm}

For each of the $n$ independent trials, 
we first generate $2K$ canaries independently 
from the distribution of original training dataset $\mathcal{D}$. 
The first $K$ canaries are then appended to the $\mathcal{D}$ to form an augmented dataset $\mathcal{D}'$. 
Two versions of this augmented dataset are created: $\mathcal{D}'_0$ and $\mathcal{D}'_1$, 
differing slightly in the applied encoding schemes $\phi_1$ and $\phi_2$, 
which perturb the angle encoding by an offset parameter.
Both datasets are fed into the QML model to initialize two models 
$\theta_0$ and $\theta_1$, respectively. 
During evaluation, 
all $2K$ canaries are first mapped through 
the same quantum feature encoding ($\phi_1$ or $\phi_2$) for fair comparison.
The first $K$ encoded quantum canaries are then input to model $\theta_0$, 
while the remaining $K$ quantum canaries are input to model $\theta_1$.
For each input, we compute the corresponding loss value $l(\theta, c)$,
reflecting the model’s prediction error on the respective quantum canary $c$,
and compare it with a predefined threshold $\kappa$. 
This induces a per quantum canary rejection set $R(c)={\theta, l(\theta, c) \leq \kappa}$;
the binary decisions $\mathbb{I}$ is the indicator (memorized or not). 
We then aggregate $\mathbb{I}$ across canaries to form the audit statistic, 
which quantifies potential quantum memorization and privacy leakage.

%
%
%
%By forwarding the canaries through the trained QML models, 
%we obtain the associated loss values, 
%which are subsequently compared against a predefined threshold $\kappa$. 
%These comparisons yield binary test statistics for each quantum canary, 
%recording whether it is memorized by the QML model.

To achieve statistical validity, each trial records two sets of binary indicators 
corresponding to whether the canary-injected model 
and the baseline model satisfy the loss-based detection criterion. 
Following $n$ trials, these binary indicators are aggregated 
and processed through specially designed adaptive confidence interval estimators, 
namely XBernLower and XBernUpper,
which are based on the \texttt{eXchangeable Bernoulli} (XBern) distribution~\cite{pillutla2024unleashing}.
Finally, an empirical privacy loss bound $\hat{\epsilon}_n$ is computed according to the Lifted QDP condition, 
ensuring that with high probability at least $1-\beta$, 
the true privacy loss $\epsilon$ is bounded by $\hat{\epsilon}_n$.

We compare the sample complexity 
and of our proposed lifted QDP for auditing algorithm (\Cref{alg1}) 
with the standard QDP-based auditing algorithm, 
where each round inserts a single quantum canary and performs one test,
which are listed in \Cref{tab:complexity_comparison}.

\begin{table}[htbp]
\centering
\caption{Sample Complexity Comparison between QDP-based and Lifted QDP-based Auditing}
\label{tab:complexity_comparison}
\begin{tabular}{c|c}
\toprule
\textbf{Method} & \textbf{Sample Complexity} \\
\midrule
QDP & $\mathcal{O}\left( \frac{\ln(1/\beta)}{\Delta^2} \right)$ \\
Lifted QDP & $\mathcal{O}\left( \frac{\ln(1/\beta)}{K\Delta^2} \right)$  \\
\bottomrule
\end{tabular}
\end{table}

Sample complexity refers to the number of independent train-and-test runs 
required to obtain a reliable lower bound on $\epsilon$ with confidence $1-\beta$.
For QDP based auditing,
each run we record a Bernoulli random variable $Z=1[\operatorname{loss}< \kappa]$,
where $\kappa$ is the threshold.
Under the hypotheses canary present and canary absent,
the success probabilities are $p_{1}$ and $p_{0}$.
Their gap $\Delta =|p_{1}-p_{0}|$
is the statistical signal exploited by the audit.
Applying Hoeffding--Chernoff concentration to the empirical mean $\bar Z$ gives the classical bound $n=\mathcal{O}\!\Bigl(\tfrac{\ln(1/\beta)}{\Delta^{2}}\Bigr)$.
For Lifted QDP,
a single run now injects $K$ canaries simultaneously and declares
success only if all $K$ losses fall below $\kappa$.
Assuming the $K$ canaries are approximately independent,
in Lifted QDP, sampler complexity is $\mathcal{O}\left( \frac{\ln(1/\beta)}{K\Delta^2} \right)$
and the theoretical sample complexity is reduced by a factor of $K$.
There is noted that the saving is an \emph{upper bound} attainable,
and in QML experiments,
the extra optimization overhead (memory, longer convergence) reduce the effective gap, so the real-world speed-up largely falls short of the ideal one.

%% file: tex/experiment.tex
\section{Experiments}\label{sec-qmldp}

In this section, 
we evaluate the effectiveness and efficiency 
of our proposed \textit{Lifted QDP for Auditing} algorithm 
by comparing it against standard QDP auditing methods. 
Specifically, 
we demonstrate that our method achieves a more accurate bound on the privacy parameter $\epsilon$, 
offering estimates that are closer to theoretical guarantees. 
Furthermore, we validate its efficiency 
by reporting a significant reduction in runtime compared to baseline approaches.
\Cref{data} outlines the datasets and experimental settings used in our evaluation.
\Cref{performance} presents the auditing performance of \textit{Lifted QDP for Auditing}, 
highlighting both accuracy and computational gains.
In \Cref{parameters}, we further investigate how key parameters affect performance.
Finally, we summarize key findings 
and provide insights into the privacy behaviors of QML models in \Cref{finding}.
%we demonstrate its effectiveness by 
%showing that it achieves a lower bound $\epsilon$ 
%that is closer to the theoretical optimum. 
%Furthermore, we validate its efficiency by 
%reporting significantly reduced runtime compared to baseline methods.

\subsection{Dataset and Setting}\label{data}

%We selected two representative datasets, 
%the Iris dataset~\cite{tomal2024quantum} 
%and the Genomic Benchmarks (Human or Worm dataset)~\cite{singh2025modeling}, 
%to evaluate the effectiveness and efficiency of our proposed algorithm
%using two hybrid quantum classical models,  the QNN and VQSVM.
%For the VQSVM model, 
%we conducted binary classification experiments 
%using two classes from the Iris dataset (IRIS-VQSVM). 
%The QNN model was applied to both datasets, denoted as IRIS-QNN and GB-QNN. 
We selected three widely used datasets for our experiments: the Iris dataset~\cite{tomal2024quantum}, 
Genomic Benchmarks (specifically the Human and Worm subsets)~\cite{singh2025modeling}, 
and MNIST~\cite{senokosov2024quantum}. 
To enable repeated statistical evaluations and reduce computational complexity, 
we simplified the Genomic Benchmarks and MNIST datasets. 
For Genomic Benchmarks, we used 200 training samples per class. 
For MNIST, we selected only digits 0 and 1, 
using 50 samples from each class for training.

All quantum circuits employed angle encoding for feature mapping 
and the \textit{RealAmplitudes} ansatz with a linear entanglement strategy. 
The number of qubits was determined based on dataset complexity: 
4 qubits for Iris, 5 for Genomic Benchmarks, 
and 8 for MNIST.
To simulate realistic quantum conditions, we evaluated the models under two types of noise: 
depolarizing noise and measurement noise. 
The strength of the depolarizing noise is controlled by a probability parameter $p$, 
which determines the likelihood of a quantum operation 
being affected by the channel. 
Measurement noise, on the other hand, is governed by the number of measurement shots $s$, 
where a larger $s$ reduces the statistical uncertainty and mitigates the impact of sampling errors.
We explicitly inserted identity  gates into the quantum circuits, 
one per qubit, to simulate exposure to quantum noise while preserving the circuit logic. 
% The performance of Lifted QDP is illustrated in \Cref{fig:threeplots}, 
% \Cref{fig:k}, and \Cref{time}. 

\subsection{The Auditing Performance}\label{performance}

To demonstrate the effectiveness of the proposed Lifted QDP-based auditing algorithm, 
we compare its empirical estimation of the privacy parameter $\hat{\epsilon}$ 
with that of a baseline QDP-based auditing algorithm under identical trial settings. 
In addition, to assess the discrepancy 
between the actual privacy budget and its theoretical guarantee,
we introduce a another baseline based on the theoretical upper bound of $\epsilon$.
This theoretical value is proved in ~\cite{guan2023detecting},
which shows that white-box auditing method
tends to yield estimates approaching the theoretical maximum privacy loss.
As illustrated in ~\Cref{fig:sixplots}, 
each line corresponds to a different experimental configuration. 
The red curve (QDP) represents the baseline approach, 
where each trial involves designing and testing a single quantum canary.
The blue curve corresponds to the lifted QDP-based auditing method, 
where each trial incorporates $K$ quantum canaries simultaneously. 
% Both methods estimate the lower bound of $\hat{\epsilon}$ based on model behavior, rather than relying on internal noise models. 
The gray line indicates the theoretical upper bound of $\epsilon$.

\begin{figure}[htbp]
	\centering
	\subfloat[IRIS Dataset with Depolarizing Noise ($K=32$)]{ 
		\includegraphics[width=0.43\linewidth]{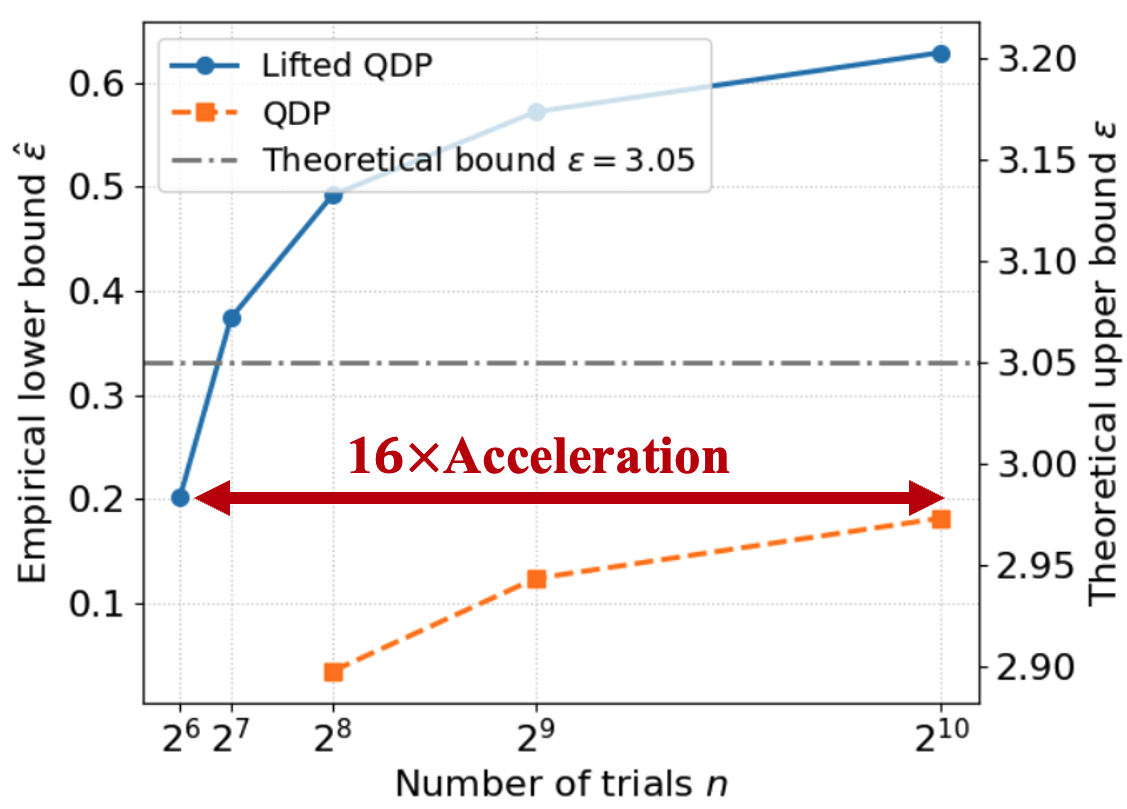}
	}
	\hspace{10pt}
	\subfloat[IRIS Dataset with Measurement Noise ($K=32$)]{ 
		\includegraphics[width=0.43\linewidth]{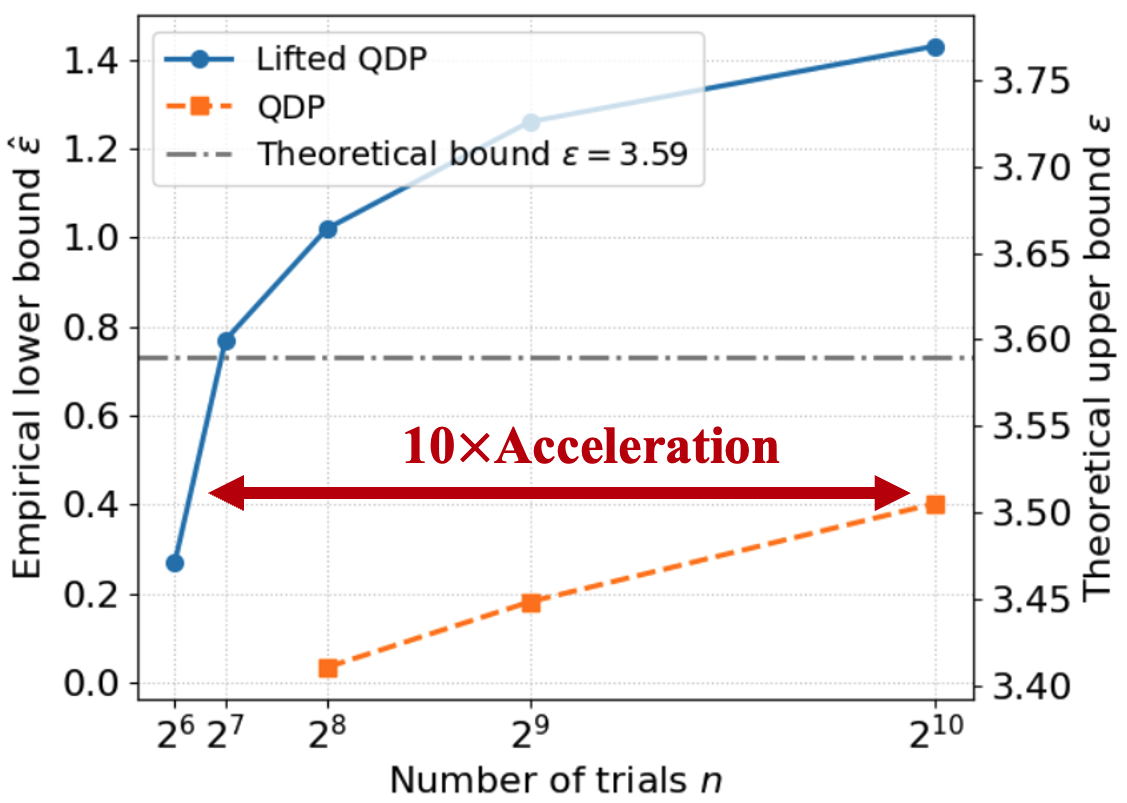}
	}

	\vspace{10pt}
	\subfloat[GB Dataset with Depolarizing Noise ($K=64$)]{ 
		\includegraphics[width=0.43\linewidth]{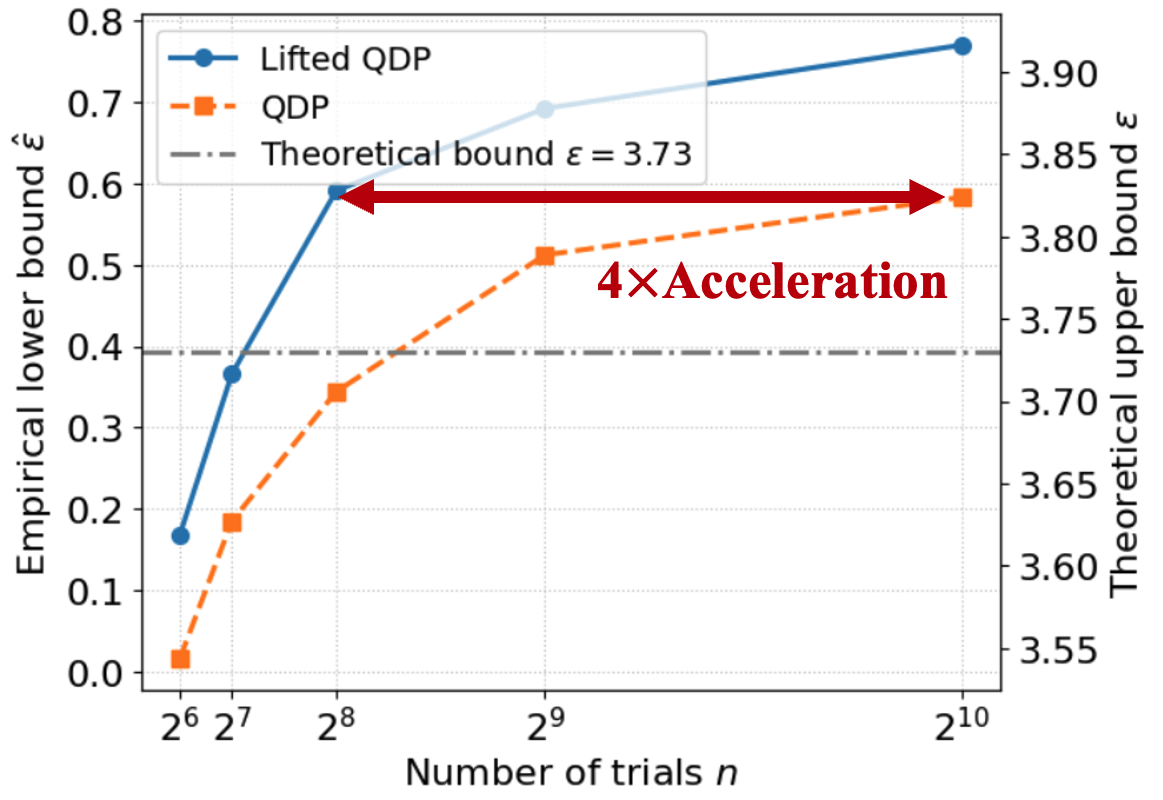}
	}
	\hspace{10pt}
	\subfloat[GB Dataset with Measurement Noise ($K=64$)]{ 
		\includegraphics[width=0.43\linewidth]{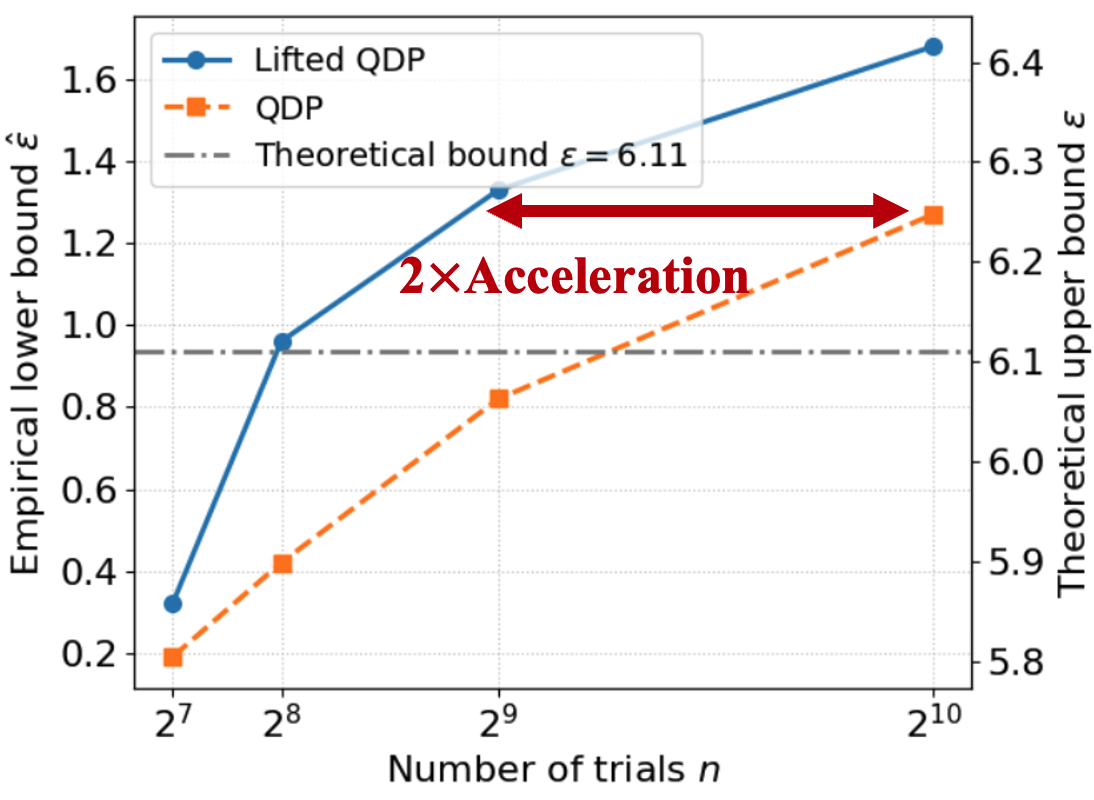}
	}

	\vspace{10pt}
	\subfloat[MNIST Dataset with Depolarizing Noise($K=16$)]{ 
		\includegraphics[width=0.43\linewidth]{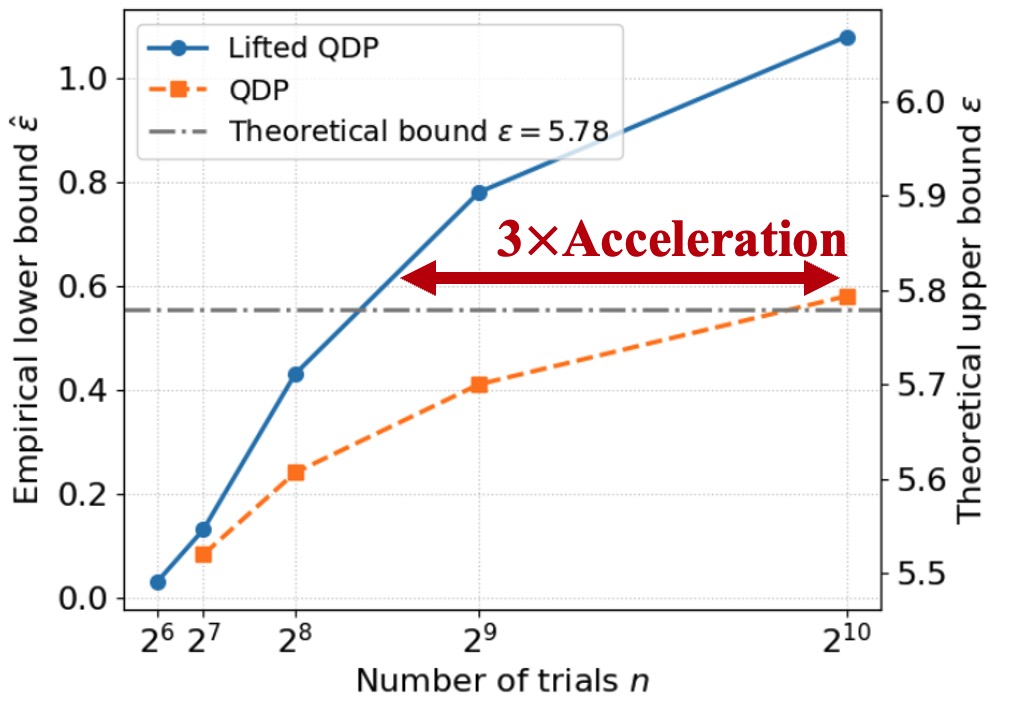}
	}
	\hspace{10pt}
	\subfloat[MNIST Dataset with Measurement Noise($K=16$)]{ 
\includegraphics[width=0.43\linewidth]{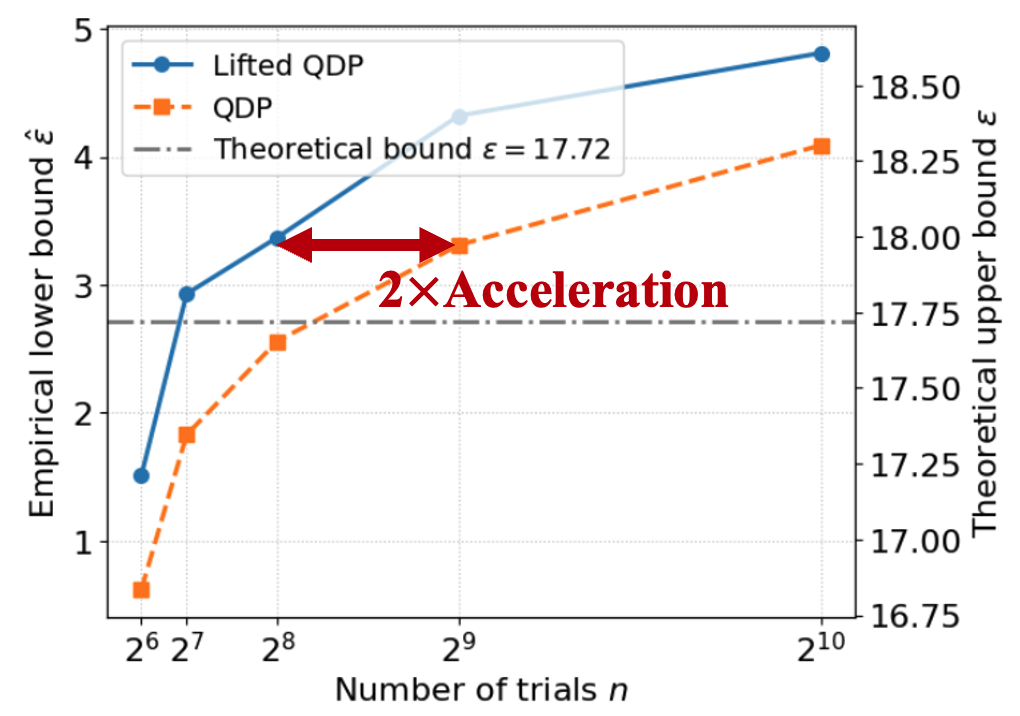}
	}
	\vspace{5pt}
	\caption{The performance of Lifted QDP auditing under different noise and models}
	\label{fig:sixplots}
\end{figure}

From ~\Cref{fig:sixplots}, 
we first observe a substantial gap between 
the theoretical and empirical values of $\epsilon$ in QML models. 
This discrepancy can be partially attributed to the structural constraints 
and parameter sparsity inherent in quantum models. 
For example, certain variational quantum circuits (VQCs) operate within low-dimensional parameter spaces, 
which limits their expressive capacity. 
At the same time, such sparsity may induce a natural "compression" effect on individual information, 
thereby reducing the model's sensitivity to injected perturbations. 
As a result, 
QML models often exhibit stronger empirical robustness against privacy leakage 
compared to the theoretical bound (worst-case assumptions).
Notably, the empirical–theoretical gap in 
$\epsilon$ is larger on \textsc{MNIST}, 
possibility due to the dimensionality reduction 
and subsequent mapping to only eight qubits, 
which compresses the input and reduces the sensitivity.

In addition, when comparing the Lifted QDP-based auditing method with the standard QDP approach, 
we observe that the number of trials 
required to reach the same estimated privacy value $\hat{\epsilon}$
can be reduced by up to 16× and at minimum by 2×,
highlighting the substantial efficiency gains enabled by the lifted design.
The actual degree of acceleration depends on both the data distribution 
and the expressive capacity of the underlying quantum circuit.
%find that 
%achieving the same estimated value $\hat{\epsilon}$ requires up to 4× fewer trials under lifted QDP. 
%This highlights the significant improvement in auditing efficiency brought by the lifted design. 
Moreover, as the number of trials $n$ increases, 
the estimated $\hat{\epsilon}$ also grows and the rate of increase gradually diminishes. 
This is because, in the early stages, additional trials substantially enhance the identification of model behavior 
and improve the chance of capturing hidden leakage. 
However, as $n$ becomes large, 
the marginal impact of each new trial diminishes, 
leading the estimate $\hat{\epsilon}$ to converge.
Finally, we compare the effects of depolarizing noise and measurement noise
under the auditing framework across three datasets.
We find that for a fixed $\hat{\epsilon}$, 
depolarizing noise  achieves acceleration with fewer trials 
than measurement noise in the majority of settings.
%
%depolarizing noise more effectively accelerates convergence in most cases.
This is likely because depolarizing noise acts uniformly on the quantum state,
which can amplify subtle differences between neighboring quantum encodings,
thereby improving the detectability of privacy leakage signals during auditing.

To further demonstrate the efficiency of the lifted QDP auditing algorithm, 
we recorded the runtime of models
under varying numbers of test canaries, 
as shown in~\Cref{time}. 
Each model is with the number of test canaries $K$ matching the configurations used in~\Cref{fig:sixplots}.
It can be observed that, 
across different values of $n$, 
the runtime of the QDP increases substantially compared to the Lifted QDP. 
For the Iris dataset ($K=32$), 
the Lifted QDP achieves a runtime speedup of approximately 26× over the standard QDP.
For MNIST ($K=16$),
we observe about a 10× speedup.
For the Genomic Benchmarks ($K=64$), 
the speedup is even more pronounced:
around 30× under depolarizing noise and up to 50× under measurement noise.
The difference in speedup across noise types is mainly driven 
by times of measurement and quantum circuit complexity. 
When circuits are relatively simple (fast to train) but involve many measured qubits 
and a large number of shots, 
measurement noise tends to deliver more pronounced speedups.
%The difference in speedup across noise types 
%is primarily due to increased measurement overhead 
%as the number of qubits grows,
%a factor that amplifies the efficiency advantage of auditing all qubits simultaneously, 
%as enabled by the lifted framework.
Besides,
while the theoretical maximum speedup is expected to scale with  $K$,
the observed runtime improvements are slightly smaller.
This is because increasing $K$ also expands the augmented training dataset,
which in turn increases the per-model training time.
As a result, 
the actual speedup factor is less than the idealized $K$-fold improvement.
%In the Lifted QDP, 
%the additional training overhead depends on the value of $K$, 
%which influences both the number of test canaries per trial 
%and the size of trained data,
%which affect the time for per model training time. 
%Notably, the performance gain from Lifted QDP is more obvious 
%in the VQSVM model than in the QNN model. 
%This is primarily because VQSVM training is inherently slower due to the need to compute and store pairwise kernel matrices, 
%and it is more sensitive to the number of training samples, 
%exhibiting poorer scalability with increasing dataset size.

\begin{table}[htbp]
\footnotesize
\centering
\caption{Runtime comparison (in minutes) between QDP and Lifted QDP}
\label{time}
\begin{tabularx}{\linewidth}{c|c*{4}{>{\centering\arraybackslash}X}}
\toprule
\textbf{Dataset-Noise} & \textbf{Method} & $2^{8}$& $2^{9}$ & $2^{10}$ & $2^{11}$   \\
\midrule
\multirow{2}{*}{IRIS-Depolarizing}
& QDP         &198.90 & 389.47 & 780.51 & 1560.59                       \\
& Lifted QDP  & 7.91 & 15.42  & 30.69   & 60.43    \\
\midrule
\multirow{2}{*}{IRIS-Measurement}
& QDP         & 106.07 & 212.88 & 425.97 & 851.52                         \\
& Lifted QDP  & 4.18   & 8.35  & 16.73  & 33.39   \\
\midrule
\multirow{2}{*}{GB-Depolarizing}
& QDP         & 500.62  & 1099.59  & 2199.20 & 4398.92  \\
& Lifted QDP  & 16.60    & 32.38   & 65.88   & 130.59      \\
\midrule
\multirow{2}{*}{GB-Measurement}
& QDP         & 399.60  & 799.11  & 1598.47 & 3196.83  \\
& Lifted QDP  & 7.68    & 14.99   & 29.60   & 58.77      \\
\midrule
\multirow{2}{*}{MNIST-Depolarizing}
& QDP         & 1889.66 & 3780.45 & 7562.43 & 15127.11\\
& Lifted QDP  & 202.02   & 403.89   & 808.23   & 1616.98    \\
\midrule
\multirow{2}{*}{MNIST-Measurement}
& QDP         & 924.20  & 1848.53  & 3698.72 & 7397.28   \\
& Lifted QDP  & 68.01    & 136.73   & 272.77   & 544.13      \\
\bottomrule
\end{tabularx}
\end{table}

\textbf{Validation on Real Quantum Hardware}: 
in order to further validate the practicality of our algorithm, 
we also used the IRIS dataset and executed the IRIS auditing workflow
on a real quantum device via the IBM Quantum platform. 
Specifically, we deployed our implementation on the \texttt{ibm_brisbane} backend.
This backend offers mid-scale quantum computation capabilities 
with support for dynamic circuits and high-fidelity single- and two-qubit gates,
making it suitable for executing parameterized quantum circuits under realistic hardware constraints.
In addition to standard depolarising and measurement noise, 
the \texttt{ibm_brisbane} device introduces crosstalk, flux-bias drift, and read-out nonlinearities, 
sources that are notoriously hard to model faithfully in simulation. 
By capturing these hardware-specific imperfections, 
the experiment offers a more realistic assessment of how much memorisation signal survives in practice.
Due to runtime limitations imposed by the platform, 
we only performed a small-scale audit involving $k = 16$ quantum canaries 
and $\num{3}$ independent training runs,
each repeated over 3 training epochs.
The training and inference process followed our lifted QDP auditing framework, 
using offset-encoded quantum canaries.
The output predictions on the canaries were collected after each training run 
and statistically aggregated to estimate the empirical privacy leakage.

In this setting, 
the estimated leakage bound in this setting was $\hat{\varepsilon} = 0.031$.
This small result has two main reasons. 
First, real quantum hardware introduces additional noise sources beyond depolarizing and measurement noise, 
which may obscure subtle leakage signals. 
Second, the small values of $K$ and $n$ used in the test limit the statistical power of the audit, 
naturally resulting in a lower estimate of $\hat{\epsilon}$.
Moreover, the use of only $\num{3}$ training epochs may not sufficiently amplify memorization behaviors in the model. 
Nonetheless, this experiment demonstrates the deployability of our lifted QDP auditing algorithm 
in practical quantum computing environments.

\subsection{Parameter Sensitivity Analysis}\label{parameters}
Beyond the number of trials $n$, 
three additional parameters play a critical role in shaping the auditing performance of the Lifted QDP: the number of quantum canaries per trial $k$,
the input encoding offset $\sigma$,
and the noise magnitude (the depolarizing probability $p$ in depolarizing noise, 
and the number of measurement shots $N$ in measurement noise).
These parameter effects are illustrated in \Cref{varyk}, \Cref{varysigma}, and \Cref{varyn}, respectively.

\begin{figure}[htbp]
	\centering 
	\subfloat[IRIS]{ 
		\label{kfig:subfig1}
		\includegraphics[width=0.65\linewidth]{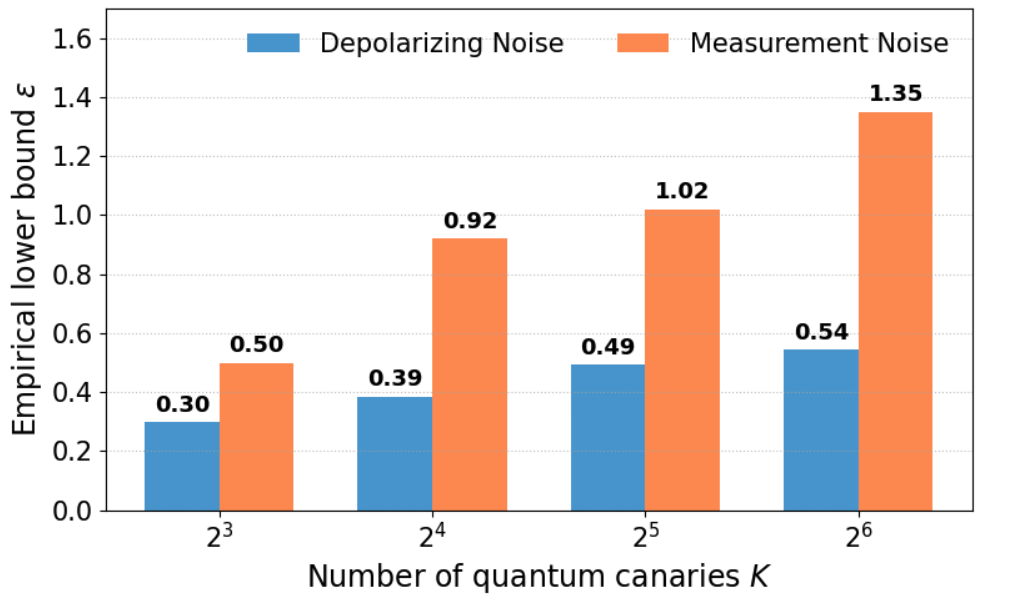}} 
	\hspace{5pt}
	\subfloat[Genomic Benchmarks]{ 
		\label{fig:subfig2}
		\includegraphics[width=0.65\linewidth]{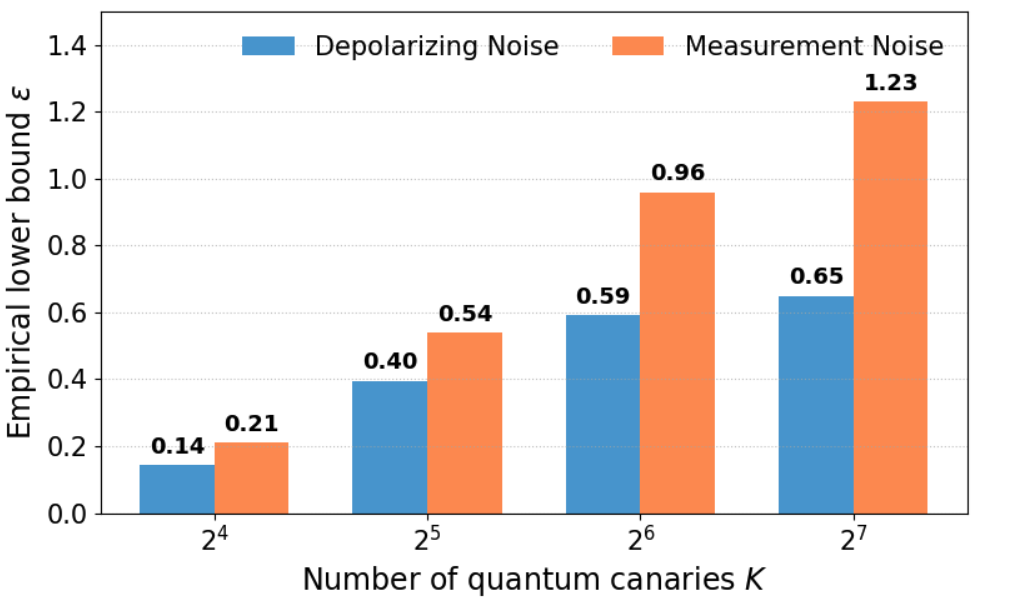}} 
	\hspace{5pt}
	\subfloat[MNIST]{ 
		\label{fig:subfig3}
		\includegraphics[width=0.65\linewidth]{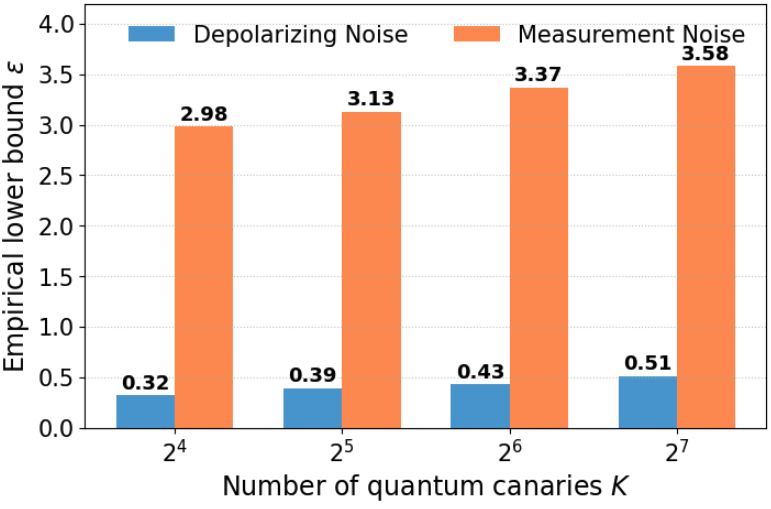}} 
	\vspace{5pt}
	\caption{The bound of $\hat{\epsilon}$ by varying $K$} 
	\label{varyk}
\end{figure}

\textbf{The number of quantum canaries $K$}: 
\Cref{varyk} illustrates the estimated values of $\hat{\epsilon}$ 
across the three models under both depolarizing and measurement noise, 
with varying numbers of quantum canaries per trial $K$. 
The depolarizing probability was fixed at $p = 0.01$, 
the number of measurement shots was set to $N = 32$, 
and the number of trials was set to $n = 256$.
We observe that as $k$ increases, 
the estimated $\hat{\epsilon}$ also increases, but the rate of growth gradually diminishes. 
This trend can be attributed to the fact that when $K$ is small, 
adding each additional quantum canary significantly amplifies the probability difference, 
thereby increasing the measured privacy budget. 
However, as $k$ continues to grow, 
the event space begins to saturate success probabilities approach 0 or 1,
leading to diminishing marginal contributions from each added quantum canary. 

In addition, we observe that in \Cref{kfig:subfig1},
when the theoretical privacy bounds for both noise types are close,
$\hat{\epsilon}$ under measurement noise are consistently higher than those under depolarizing noise. 
%This effect is particularly evident in \Cref{kfig:subfig1}, 
%where the theoretical privacy bounds for both noise types are close, 
%yet the observed leakage differs significantly.
This suggests that, 
for a given privacy budget, 
applying depolarizing noise results in lower actual privacy leakage 
compared to measurement noise.
One possible explanation is that depolarizing noise introduces uniform perturbations in quantum channel, 
which may obscure fine-grained model behavior more effectively. 
In contrast, 
measurement noise affects only the readout layer 
and may leave deeper model-level memorization signals intact, 
thus leading to greater effective leakage.
These findings highlight the importance of noise model selection in QML privacy protection: 
even when theoretical bounds are comparable, 
the practical leakage behavior can vary considerably, 
and depolarizing noise may offer stronger empirical privacy guarantees 
under the same theoretical constraints.
% This reflects a saturation effect in detectability, where further increasing $k$ yields smaller improvements in auditing resolution.

%\begin{figure}[htbp]
%	\centering 
%	\subfloat[Depolarizing Noise]{ 
%		\label{fig:subfig1}
%		\includegraphics[width=0.43\linewidth]{figure/sigma.jpg}} 
%	\hspace{5pt}
%	\subfloat[Measurement Noise]{ 
%		\label{fig:subfig2}
%		\includegraphics[width=0.43\linewidth]{figure/sigma22.jpg}} 
%	\vspace{5pt}
%	\caption{The relationship between the $\hat{\epsilon}$ and $\sigma$} 
%	\label{varyn}
%\end{figure}

\textbf{The encoding offset $\sigma$}: 
We further investigate how the input encoding offset $\sigma$ 
affect the estimated lower bound of $\varepsilon$.
To begin with, 
we visualize how different values of $\sigma$ influence the angle distribution of individual qubit after encoding, 
as shown in \Cref{visual}.
Here, the encoding scheme $\phi_1$ refers to direct mapping,
while $\phi_2$ applies an additive angular offset 
sampled from a Gaussian distribution with standard deviation $\sigma$.
As illustrated, when $\sigma$ is small,
the distributions of qubit angles under 
$\phi_1$ and $\phi_2$ remain almost indistinguishable,
indicating negligible perturbation at the encoding level.
As $\sigma$ increases,
the angular spread under $\phi_2$ becomes significantly broader,
resulting in larger trace distances between the quantum states 
produced by the two encoding schemes.
This growing separation directly impacts the level of privacy leakage observed:
since greater trace distance corresponds to larger distinguishability 
between neighboring quantum states,
the privacy budget also increases accordingly.

\begin{figure}[htbp]
	\centering 
	\subfloat[$\sigma=0.01$]{ 
		\label{vfig:subfig1}
		\includegraphics[width=0.45\linewidth]{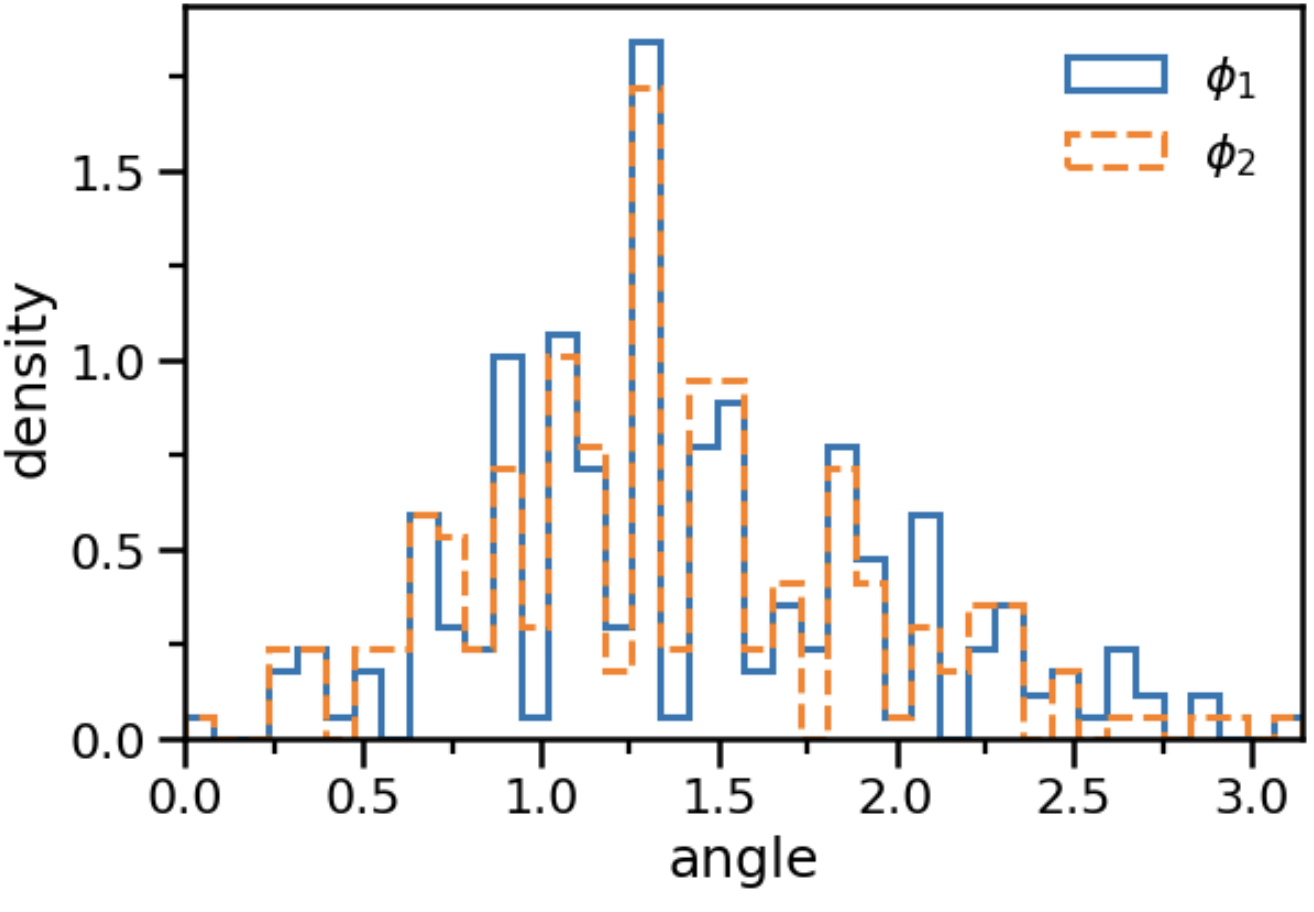}} 
	\hspace{5pt}
	\subfloat[$\sigma=0.05$]{ 
		\label{vfig:subfig2}
		\includegraphics[width=0.45\linewidth]{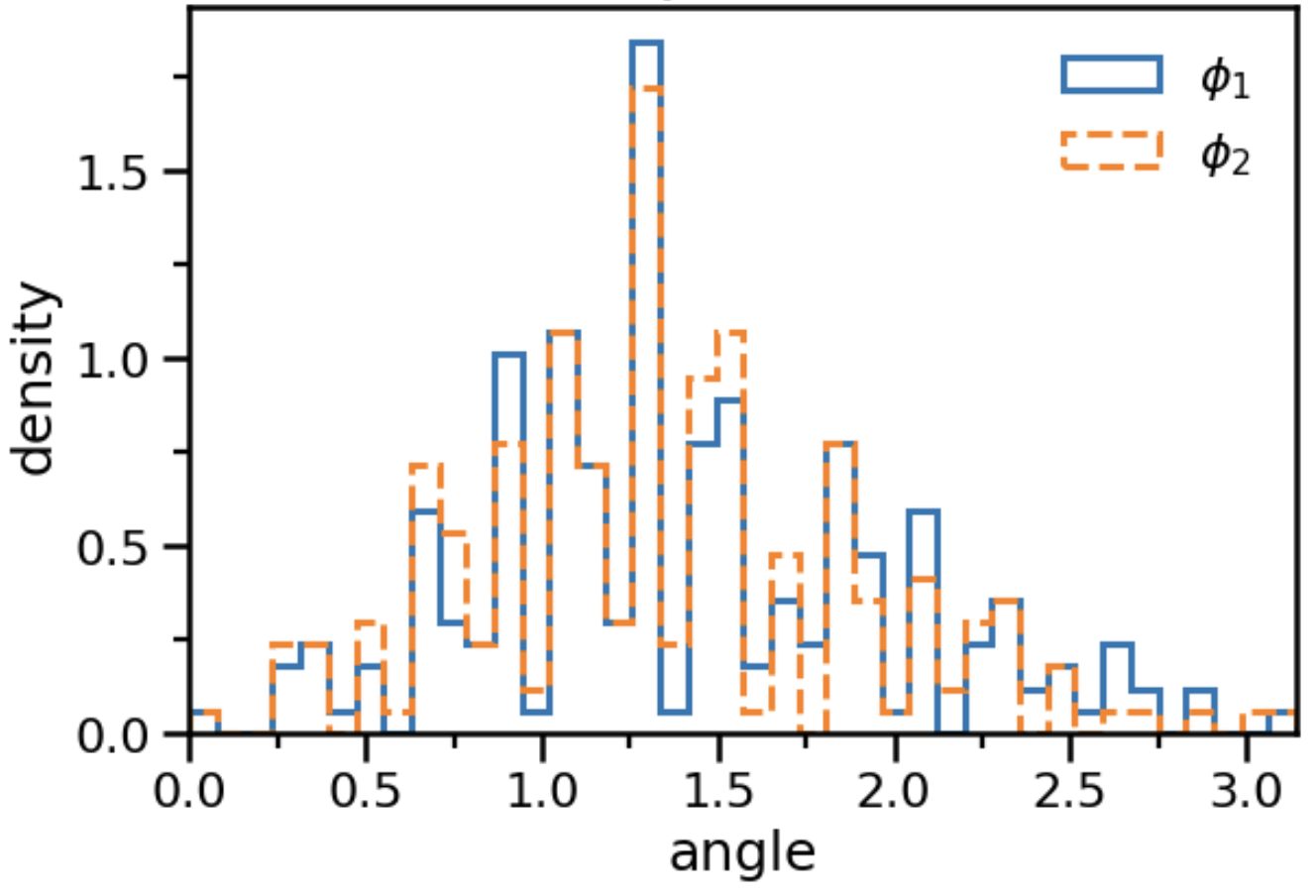}} 
	\hspace{5pt}
	\subfloat[$\sigma=0.1$]{ 
		\label{vfig:subfig3}
		\includegraphics[width=0.45\linewidth]{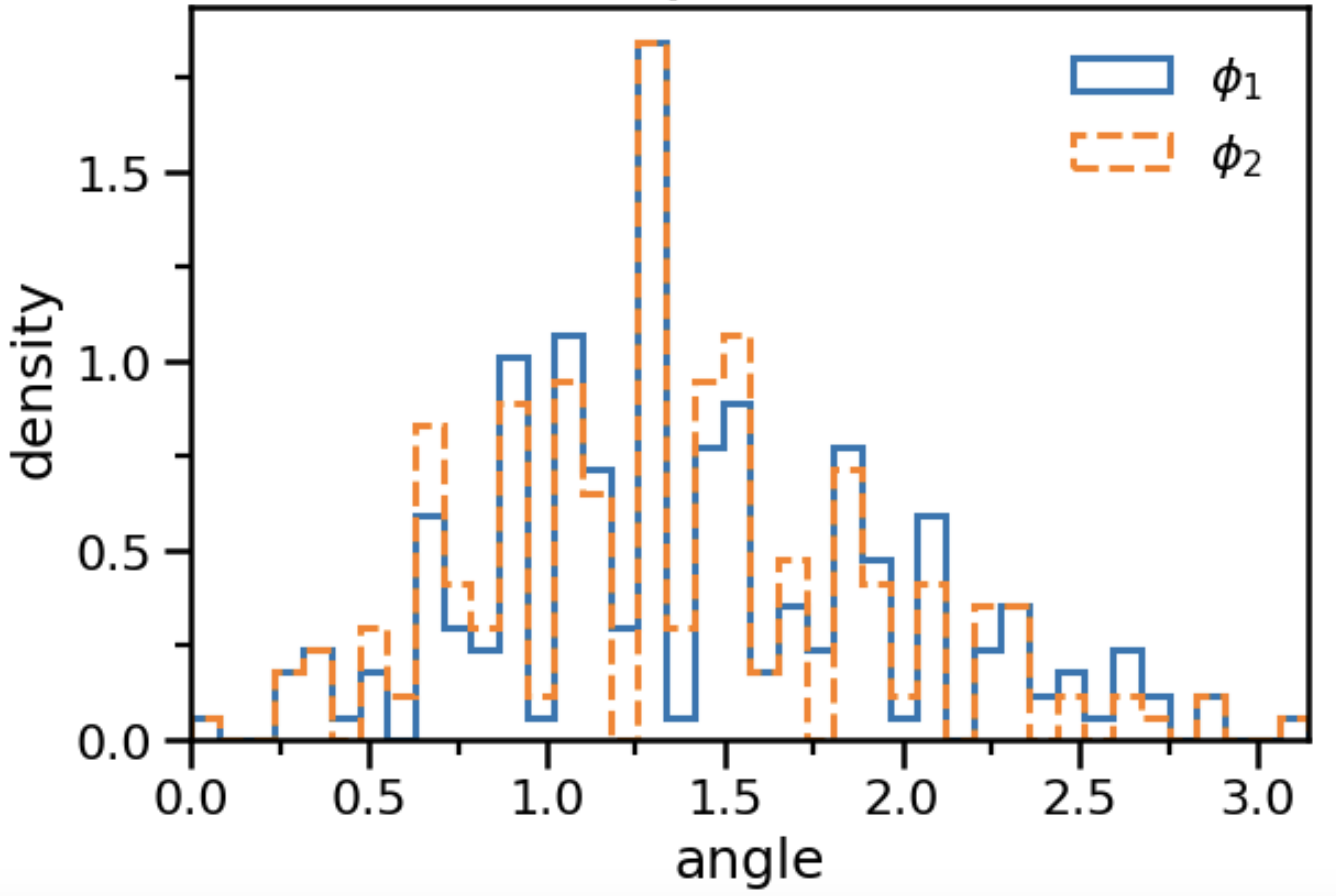}} 
	\vspace{5pt}
	\caption{Distribution of qubit angles after different encoding schemes} 
	\label{visual}
\end{figure}

In \Cref{varyn},  
increasing the value of $\sigma$ always pushes the empirical estimate $\hat{\epsilon}$ upward 
across different noise settings.
According to \Cref{the:the2}, 
$\sigma$ directly influences the bound on the trace distance between encoded quantum states, 
which in turn affects the privacy budget bound. 
Moreover, 
the increase effect becomes more evident when the shot number  $n$ is large.
When $n$ is small, 
the count statistics are coarse, 
so $\hat{\epsilon}$ is systematically under-estimated and different $\sigma$ values are partially “masked” by the large sampling variance.
As $n$ grows, 
the estimator approaches its asymptotic (unbiased) regime, and the contribution of $\sigma$ emerges clearly, 
causing the gaps between curves to widen.

\begin{figure}[htbp]
	\centering
	\subfloat[IRIS Dataset with Depolarizing Noise]{ 
		\includegraphics[width=0.43\linewidth]{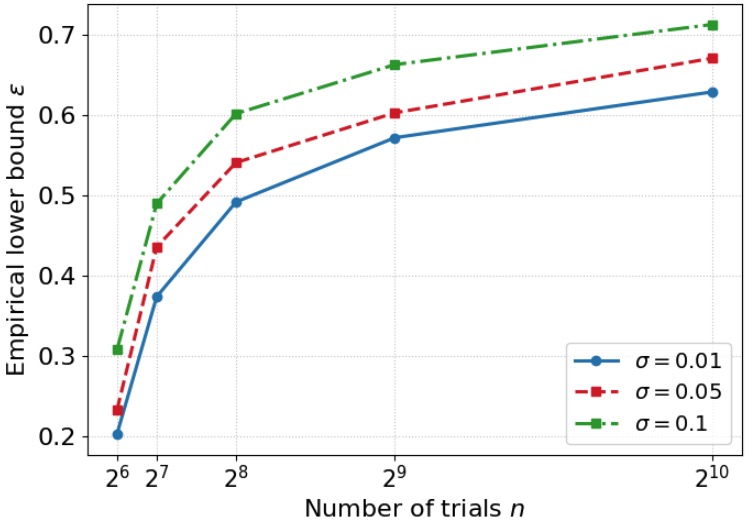}
	}
	\hspace{10pt}
	\subfloat[IRIS Dataset with Measurement Noise ]{ 
		\includegraphics[width=0.43\linewidth]{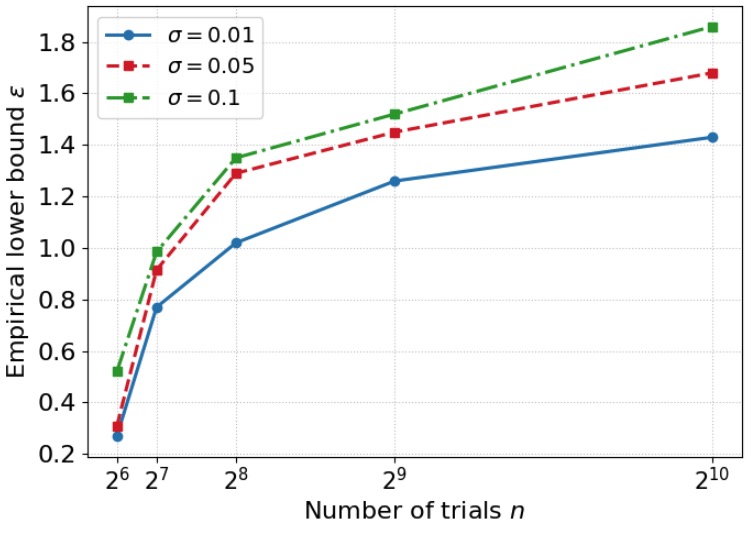}
	}
	
	\vspace{10pt}
	\subfloat[GB Dataset with Depolarizing Noise]{ 
		\includegraphics[width=0.43\linewidth]{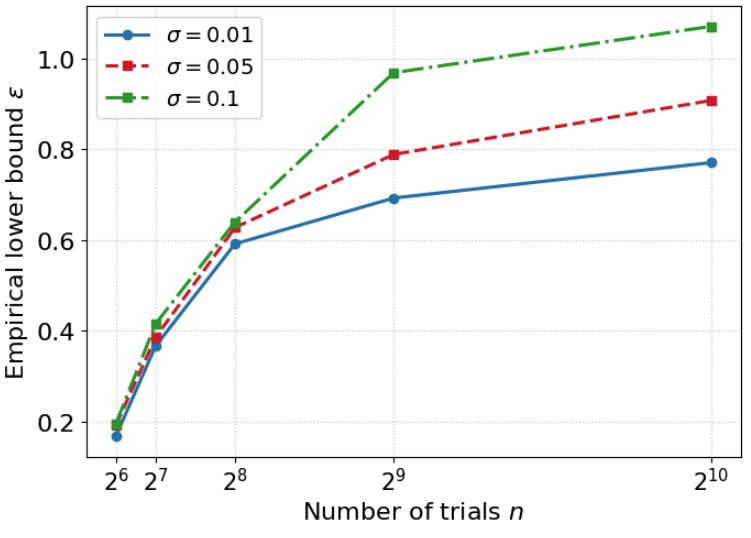}
	}
	\hspace{10pt}
	\subfloat[GB Dataset with Measurement Noise]{ 
		\includegraphics[width=0.43\linewidth]{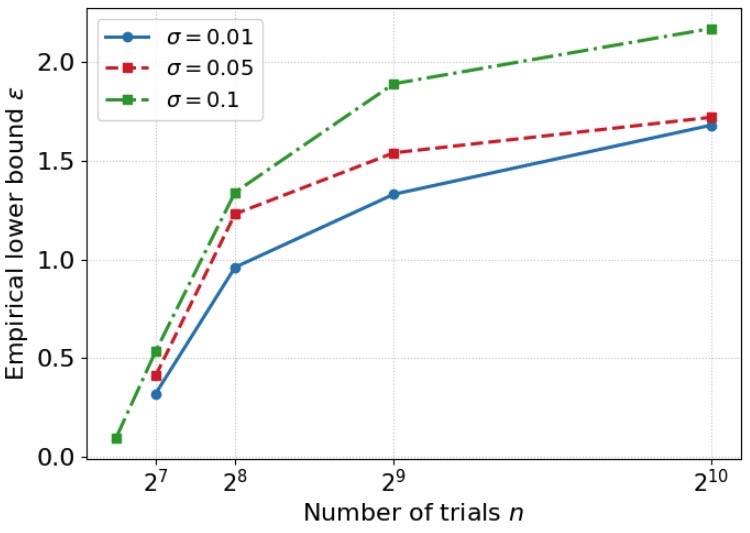}
	}
	
	\vspace{10pt}
	\subfloat[MNIST Dataset with Depolarizing Noise]{ 
		\includegraphics[width=0.43\linewidth]{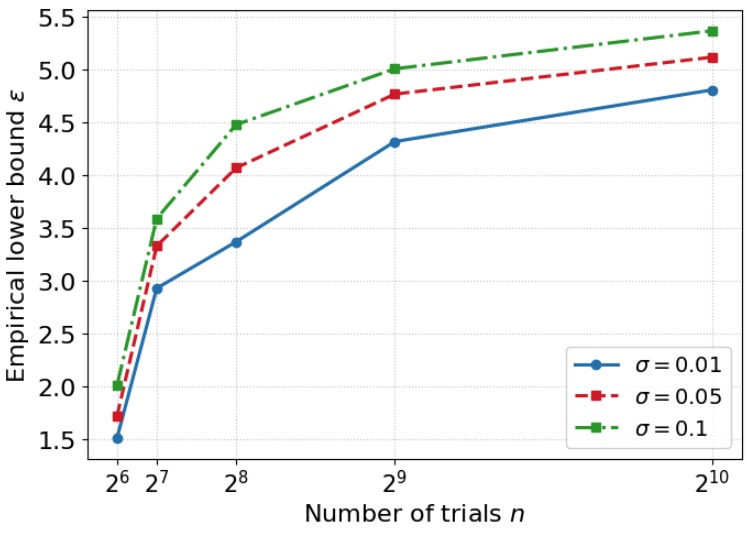}
	}
	\hspace{10pt}
	\subfloat[MNIST Dataset with Measurement Noise]{ 
		\includegraphics[width=0.43\linewidth]{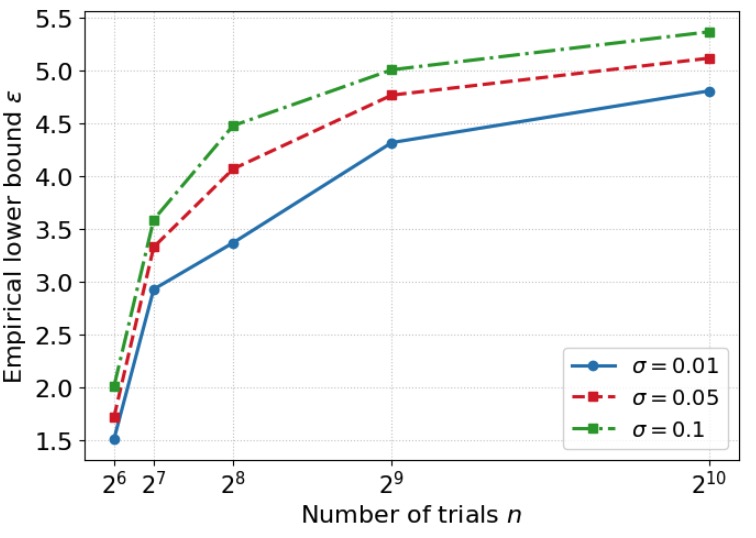}
	}
	\vspace{5pt}
	\caption{The relationship between $\hat{\epsilon}$ and $\sigma$}
	\label{varyn}
\end{figure}

\textbf{Noise magnitude ($p$ or $N$)}: \Cref{varysigma} illustrates the relationship between noise amplitude and the estimated privacy budget $\hat{\epsilon}$. 
As indicated by \eqref{equ:12} and \eqref{equ:20}, 
decreasing the depolarizing probability $p$ 
or increasing the number of measurement shots $N$ leads to a reduction in $\hat{\epsilon}$. 
This is intuitive: a smaller $p$ means a lower chance of qubit depolarization, 
while a larger $N$ reduces statistical uncertainty in measurements, 
both resulting in less noise and improved privacy protection.
In addition, we observe that when the number of trials $n$ is small, 
varying $p$ or $N$ has little impact on the estimated $\hat{\epsilon}$. 
This is because under limited sampling, the dominant source of error arises from 
statistical variance rather than noise amplitude. 
Only when $n$ is sufficiently large do the benefits of 
reducing depolarization or increasing measurement 
precision manifest more clearly in the observed privacy budget estimates.

%As shown in \Cref{varysigma}, increasing the value of $\sigma$ consistently leads to a reduction in the estimated $\hat{\epsilon}$ across different noise settings.
%It can be explained by $\sigma$ can effectively control the distance
%between perturbed and unperturbed qubits in the feature space. 
%According to \Cref{the:the2}, 
%$\sigma$ directly influences the bound on the trace distance between encoded quantum states, 
%which in turn affects the privacy budget bound. 
%Specifically, 
%a larger $\sigma$ results in treating quantum states with higher trace distance as neighboring states under the QDP definition, 
%thereby enforcing stronger privacy guarantees. 
%This leads to a tighter lower bound on privacy budget and consequently a smaller estimated $\hat{\epsilon}$.

\begin{figure}[htbp]
	\centering
	\subfloat[IRIS Dataset with Depolarizing Noise]{ 
		\includegraphics[width=0.43\linewidth]{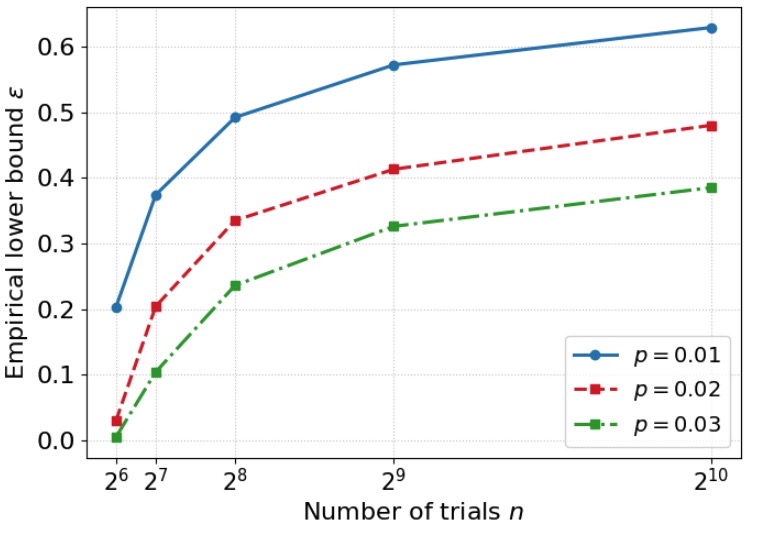}
	}
	\hspace{10pt}
	\subfloat[IRIS Dataset with Measurement Noise ]{ 
		\includegraphics[width=0.43\linewidth]{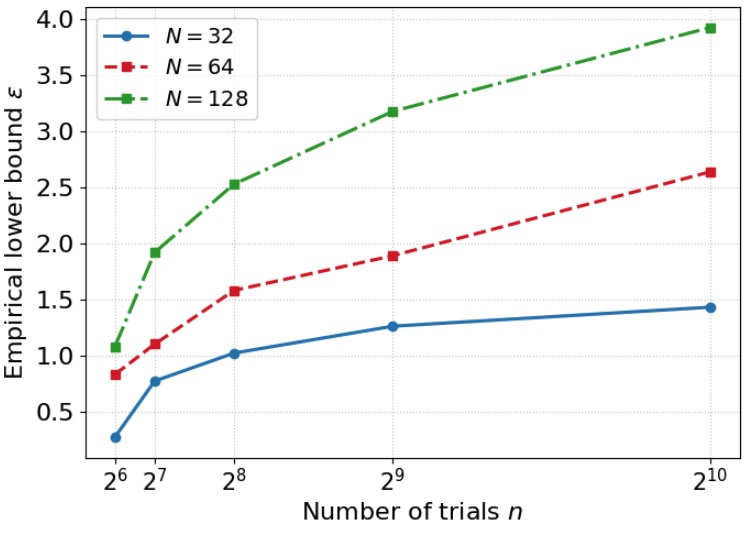}
	}
	
	\vspace{10pt}
	\subfloat[GB Dataset with Depolarizing Noise]{ 
		\includegraphics[width=0.43\linewidth]{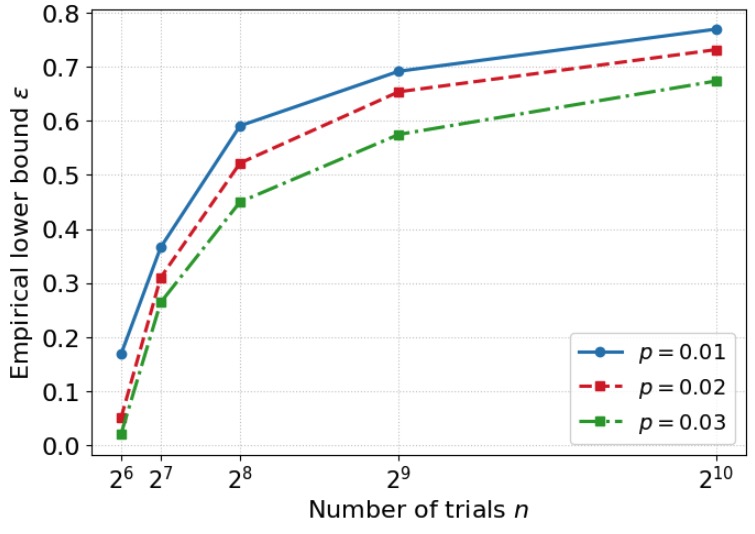}
	}
	\hspace{10pt}
	\subfloat[GB Dataset with Measurement Noise]{ 
		\includegraphics[width=0.43\linewidth]{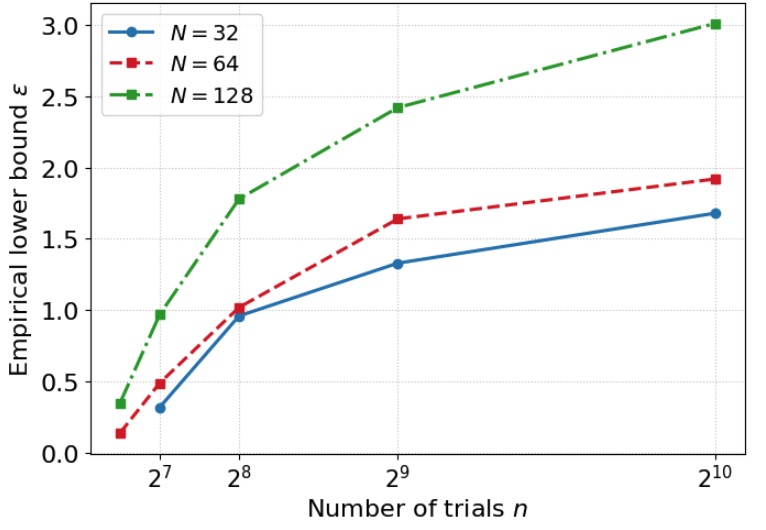}
	}
	
	\vspace{10pt}
	\subfloat[MNIST Dataset with Depolarizing Noise]{ 
		\includegraphics[width=0.43\linewidth]{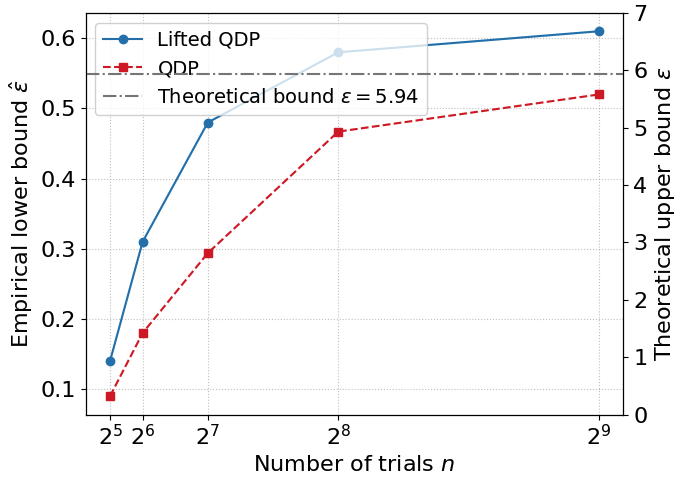}
	}
	\hspace{10pt}
	\subfloat[MNIST Dataset with Measurement Noise]{ 
		\includegraphics[width=0.43\linewidth]{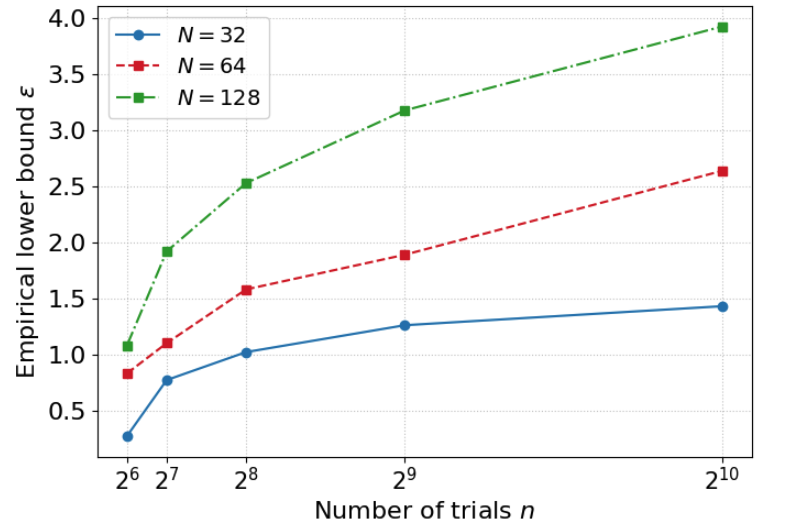}
	}
	\vspace{5pt}
	\caption{The relationship between the $\hat{\epsilon}$ and noise amplitude}
		\label{varysigma}
\end{figure}

\subsection{Finding and Discussion}\label{finding}

Our simulations and hardware runs show that QML models 
offer a built-in privacy advantage over classical models. 
In practice, QMLmodels leak less information than their theoretical privacy budgets predict. 
Three factors explain this: (i) limited circuit expressivity keeps variational quantum circuits from fitting arbitrarily complex decision boundaries; 
(ii) barren plateaus hamper optimization~\cite{mcclean2018barren}, preventing the model from memorizing the training set; 
and (iii) hardware noise injects randomness that further obscures individual records. 
Moreover, real-world QML models usually compress high-dimensional data  
before encoding it into the small number of qubits, 
discarding fine-grained details and lowering leakage risk even more. 
These properties make QML especially attractive for privacy-critical domains, such as healthcare or genomics, 
where a slight trade-off in accuracy is acceptable in exchange for stronger protection of sensitive data.

We also observe that, 
for a fixed theoretical privacy budget, 
depolarizing noise tends to yield stronger empirical privacy guarantees than measurement noise.
This finding implies that, 
when apportioning the overall privacy budget, 
we can allocate a smaller share to the depolarizing channel and reserve a larger share for the measurement stage,
thereby further reducing the risk of privacy leakage in practice.

Within the Lifted QDP auditing framework, 
introducing multiple quantum canaries can drastically shorten the time needed to estimate privacy leakage 
and produces  $\hat{\epsilon}$ values that more closely match the theoretical bound. 
It is important to emphasize, 
$\hat{\epsilon}$ depends on far more than the listed parameters,
canary count $K$,  trial number $n$,
and offset $\sigma$.
Additional factors such as circuit depth and the distribution of data
also influence the effective leakage; 
many of these details are hidden in a black-box setting, 
which is precisely why black-box auditing is necessary. 
The present experiments therefore serve only to demonstrate the feasibility and efficiency of the Lifted QDP auditing framework. 
They should not be interpreted as a universal rule 
that adding a given amount of theoretical noise will 
always produce a proportional reduction in leakage; 
each model–hardware combination must be assessed on its own terms.

%% file: tex/conclusions.tex
\section{Conclusion} \label{sec-conclusions}

We present a novel black-box auditing framework for evaluating privacy leakage in QML models. 
Unlike prior white-box approaches providing only theoretical upper bounds, 
our method leverages \emph{Lifted QDP} to estimate empirical privacy leakage using black-box access only.
At the core of our framework are quantum canaries---offset-encoded quantum states inserted into training data to probe model memorization. 
By analyzing model responses to these canaries at inference, 
we estimate lower bounds on the privacy budget $\epsilon$, 
providing interpretable, data-driven privacy assessment.
This bridges theoretical privacy guarantees with practical auditing in deployment settings where model internals are inaccessible.
We validate the framework through extensive simulations and real quantum hardware evaluations, 
demonstrating effectiveness for real-world QML privacy auditing.

%\emph{Limitations}.
Our framework has three key limitations.
First, quantum canaries use Gaussian-sampled fixed offsets independent of model dynamics, 
potentially failing to detect localized memorization or adversarial obfuscation.
Second, experiments isolate single noise sources, 
while real quantum hardware exhibits simultaneous, interacting noise channels that may alter privacy estimates.
Finally, we address only differential privacy leakage, 
omitting broader inference threats (membership, attribute, property inference).

\emph{Future Work}.
Future directions include: 
i) adaptive quantum canary generation aligned with model representations, 
ii) joint modeling of multiple noise sources under Lifted QDP, 
and iii) extension to comprehensive inference threat evaluation 
for holistic adversarial robustness assessment.